\documentclass{article} 
\pdfoutput=1
\usepackage[left=1in,top=1in,right=1in,bottom=1in,nohead,footskip=20pt]{geometry}
\usepackage{graphicx}

\usepackage{graphics} 
\usepackage{epsfig} 
\usepackage{times} 
\usepackage{amsmath} 
\usepackage{amssymb}  
\usepackage{color}
\usepackage{chemarr}
\usepackage{extarrows}
\usepackage[all]{xy}
\usepackage[pdftex, colorlinks=true, urlcolor=MyDarkBlue, citecolor=MyDarkRed, linkcolor=MyDarkGreen, bookmarks=true]{hyperref}
\usepackage{subfigure}
\usepackage{longtable}
\usepackage{float}
\usepackage{enumitem}

\usepackage{natbib}

\usepackage{calc}
\newlength{\minsqueezelength}
\newcommand{\vsqueeze}[1]{} 

\definecolor{MyPurple}{rgb}{0.5,0,0.5}
\definecolor{MyDarkRed}{rgb}{0.5,0,0.1}
\definecolor{MyDarkBlue}{rgb}{0.1,0.1,0.5}
\definecolor{MyDarkGreen}{rgb}{0.1,0.5,0.1}
\definecolor{MyRed}{rgb}{1.0,0,0}
\definecolor{MyBlue}{rgb}{0,0,1.0}
\definecolor{MyGreen}{rgb}{0,0.8,0}
\definecolor{lightgray}{rgb}{0.96,0.96,0.96}
\definecolor{darkgray}{rgb}{0.4,0.4,0.4}
\definecolor{gray}{rgb}{0.6,0.6,0.6}

\newcommand{\mysection}[1]{\section{{#1}}}  
\newcommand{\mysubsection}[1]{\subsection{{#1}}}  
\newcommand{\mysubsubsection}[1]{\subsubsection{{#1}}} 

\newcommand{\changed}[1]{{#1}} 
\newcommand{\changedJ}[1]{{#1}} 

\newcommand{\revA}[1]{{#1}}
\newcommand{\revB}[1]{{#1}}

\newcommand{\ldq}{\text{``}}
\newcommand{\rdq}{\text{''}}

\newcommand{\mymathsf}[1]{\mathsf{\mathbf{#1}}}

\newtheorem{lemma}{Lemma}
\newtheorem{corollary}{Corollary}
\newtheorem{proposition}{Proposition}
\newtheorem{theorem}{Theorem}
\newtheorem{construction}{Construction}

\newenvironment{proof}
{\begin{quote} \noindent\textit{Proof.}~}
{\newline \phantom{qed.} \hfill\ensuremath{\square}\end{quote}}

\newenvironment{proofsketch}
{\begin{quote} \noindent\textit{Sketch of proof.}~}
{\newline \phantom{qed.} \hfill\ensuremath{\square}\end{quote}}

\makeatletter
\newcommand\footnoteref[1]{\protected@xdef\@thefnmark{\ref{#1}}\@footnotemark}
\makeatother

\begin{document}

\title 
{Path Homotopy Invariants and their Application to Optimal Trajectory Planning}

\author{
	Subhrajit Bhattacharya \\
	{\small Department of Mechanical Engineering } \\  {\small and Mechanics, Lehigh University, PA.} \\
	{\small E-mail: sub216@lehigh.edu }
\and
        Robert Ghrist \\
        {\small Department of Mathematics, } \\  {\small University of Pennsylvania, PA.}\\
        {\small E-mail: ghrist@math.upenn.edu}
}

\date{}



\maketitle

\begin{abstract}
We consider the problem of optimal path planning in different homotopy classes in a given environment.
Though important in robotics applications, path-planning with reasoning about homotopy classes of trajectories has typically focused on subsets of the Euclidean plane in the robotics literature.
The problem of finding optimal trajectories in different homotopy classes in more general configuration spaces (or even characterizing the homotopy classes of such trajectories) can be difficult.
In this paper we propose automated solutions to this problem in several general classes of configuration spaces by constructing presentations of fundamental groups and giving algorithms for solving the \emph{word problem} in such groups.
We present explicit results that apply to knot and link complements in 3-space, discuss how to extend to cylindrically-deleted coordination spaces of arbitrary dimension, \changedJ{and also present results in the coordination space of robots navigating on an Euclidean plane}.
\footnote{Parts of this paper has appeared in the proceedings of the IMA Conference on Mathematics of Robotics (IMAMR), 2015.}
\end{abstract}


\mysection{Introduction}

In the context of robot motion planning, one often encounters problems requiring optimal trajectories (paths) in different homotopy classes. For example, consideration of homotopy classes is vital in planning trajectories for robot teams separating/caging and transporting objects using a flexible cable~\citep{cable:separation:IJRR:14}, or in planning optimal trajectories for robots that are tethered to a base using a fixed-length flexible cable~\citep{ICRA:14:tethered}, or in human-robot collaborative exploration problems in context of search-and-rescue missions~\citep{DARS:14:HRI}. This paper addresses the problem of optimal path planning with homotopy class as the optimization constraint. 

There is, certainly, a large literature on minimal path-planning in computational geometry (for a brief sampling and overview, see \citep{MSNew04}).
\revB{The topology of a configuration space plays a key role in solving path-planning problems. To that end, the study of the topological invariants of robot configuration spaces and the properties of motion planning problems in such configuration spaces is not new~\citep{Farber:top:complexity,COSTA:planning,Farber:moving:obstacle,cohen:tori}. However, most of this body of research has primarily focused on deduction of topological invariants of configuration/coordination spaces, addressed existence questions, have studied properties of the motion planning problem itself, but do not explicitly implement algorithms for finding optimal paths in configuration spaces of robots.}
\revA{Classification of paths based on homotopy has \revB{also} been studied in the robotics literature in two-dimensional planar workspaces using geometric methods \citep{Homotopy:Grigoriev:98,Hershberger91computingminimum}, probabilistic road-map construction \citep{homotopy:Schmitzberger:02} techniques and triangulation-based path planning \citep{Buro-TRAstar}.
While in a planar configuration space such methods can be used for determining whether or not two trajectories belong to the same homotopy class, efficient planning for least cost trajectories with homotopy class constraints is difficult using such representations even in $2$-dimensions.
To that end in prior work we developed a graph search-based method for computing shortest paths in different \emph{homology} classes in $2$, $3$ and higher dimensional Euclidean spaces with obstacles~\citep{planning:AURO:12,homology:amai:13}, and in different \emph{homotopy} classes in planar configuration spaces with obstacles~\citep{cable:separation:IJRR:14}.
}

Of course, since the problem of computing shortest paths (even for a 3-d simply-connected polygonal domain) is NP-hard \citep{CRNew87}, we must restrict attention to subclasses of spaces, even when using homotopy path constraints. \revA{In particular, we construct a discrete graph representation of a configuration space and thus reduce the problem to shortest path computation on the graph.
This representation, along with path homotopy invariants of the configuration space, allows us to compute shortest paths in different homotopy classes using graph search-based path planning algorithms.}

This paper focuses on \revA{computation of optimal paths in different homotopy classes in} two interesting and completely different types of configuration spaces: (1) knot and link complements in 3-d; and (2) cylindrically-deleted coordination spaces \citep{GL:2006}.
\revA{Compared to our prior conference publication~\citep{Homotopy-Planning:IMAMR:15}, in this paper we have (i) provided a rigorous Proposition on the presentation of the fundamental group of knot/link complements in $3$-d and hence the justification behind the proposed homotopy invariants in such spaces, (ii) have provided an explicit presentation of the fundamental group of cylindrically deleted coordination space for robots navigating on a plane, and (iii) hence have demonstrated through simulation the problem of optimal path computation in different homotopy classes for three robots navigating on a simple planar domain.}

\smallskip


\mysection{Configuration Spaces with Free Fundamental Groups}

\revA{
\mysubsection{Preliminaries}

In this section we introduce a few fundamental definitions and concepts. Each definition is accompanied by a reference to a standard text on the subject, where the reader can find more details on these topics.

\emph{Homotopy Classes of Paths:}
Consider oriented/directed curves in a topological space, $X$. Two curves $\gamma_1,\gamma_2:[0,1]\rightarrow X$ connecting the same start and end points, $q_s,q_e\in X$, are called homotopic (or belonging to the same \emph{homotopy class}) if one can be continuously deformed into the other without intersecting any obstacle (\emph{i.e.}. there exists a continuous map $\eta: [0,1]\times [0,1] \rightarrow X$ such that $\eta(\alpha,0)=\gamma_1(\alpha),~~\eta(\alpha,1)=\gamma_2(\alpha)~\forall\alpha\!\in\![0,1]$, and $\eta(0,\beta)=q_s, \eta(1,\beta)=q_e~\forall\beta\!\in\![0,1]$ \citep{Hatcher:AlgTop}).
A set of all homotopically equivalent paths (\emph{i.e.}, homotopic paths) constitute a homotopy class. We denote the homotopy class of a path $\gamma$ a $[\gamma]$.

\emph{Fundamental Group:}
The \emph{fundamental group} or the \emph{first homotopy group} of a topological space, $X$, is the set of all homotopy classes of oriented closed loops (paths with $q_s=q_e=:q_0$) in $X$ with a group structure imposed on the set as follows: i. The identity element is the class of all loops that can be contracted/homotoped to the point $q_0$; ii. The inverse of a homotopy class, $[\gamma]$, is the homotopy class of loops constituting of the same loops as in $[\gamma]$, but with reversed orientation, and is denoted as $[-\gamma]$ or $[\gamma]^{-1}$; iii. The composition (group operation) of two classes $[\gamma_1]$ and $[\gamma_2]$ is the class of loops that are obtained by concatenating a curve in $[\gamma_1]$ with a curve in $[\gamma_2]$ (\emph{i.e.}, it is the class of the loop $\gamma(t) = {\small \left\{ \begin{array}{l} \tiny \gamma_1(2t),~0\leq t < \frac{1}{2} \\ \tiny \gamma_2(2t-1),~\frac{1}{2} \leq t \leq 1 \end{array} \right.}$).

\emph{Free Group and Free Product of Groups:}
A \emph{free group} over a set of letters/symbols is the group whose elements consists of all \emph{words} constructed out of the letters and their formal \emph{inverses}, with identity element being the \emph{empty word}, and the group operation being word concatenation (with any letter juxtaposed with its inverse reducing to the identity)~\citep{scott1964group}. Given two groups, $G$ and $H$, the free product of the groups is the group of words that can be constructed with all the elements of the groups as the letters. The free product is thus written as $G*H = \{g_1 h_1 g_2 h_2 \cdots | g_i \in G, h_i\in H\}$.
}

\mysubsection{Motivation: Homotopy Invariant in $(\mathbb{R}^2-\mathcal{O})$} \label{sec:r2-simple}

We are interested in constructing computable \emph{homotopy invariants} for trajectories in a configuration space that are amenable to graph search-based path planning. To that end there is a very simple construction for configuration spaces of the form $\mathbb{R}^2-\mathcal{O}$ (Euclidean plane punctured by obstacles)~\citep{Homotopy:Grigoriev:98, Hershberger91computingminimum, tovar:sensor08, cable:separation:IJRR:14, ICRA:14:tethered}:
We start by placing \emph{representative points}, $\zeta_i$, inside the $i^{th}$ connected component of the obstacles, $O_i \subset \mathcal{O}$. We then construct non-intersecting rays, $r_1,r_2,\cdots,r_m$, emanating from the representative points (this is always possible, for example, by choosing the rays to be parallel to each other). Now, given a curve $\gamma$ in $\mathbb{R}^2-\mathcal{O}$, we construct a \emph{word} by tracing the curve, and every time we cross a ray $r_i$ from its right to left, we insert the letter ``$r_i$'' into the word, and every time we cross it from left to right, we insert a letter ``$r_i^{-1}$'' into the word, with consecutive $r_j$ and $r_j^{-1}$ canceling each other. The word thus constructed is written as $h(\gamma)$. For example, in Figure~\ref{fig:rays}, $h(\gamma) = \ldq r_1^{-1} r_4~ r_2^{-1} r_4^{-1} r_4~ r_4^{-1} r_6^{-1} \rdq = \ldq r_1^{-1} r_4~ r_2^{-1} r_4^{-1} r_6^{-1} \rdq$.
This word, called the \emph{reduced word} for the trajectory $\gamma$, is a complete homotopy invariant for trajectories connecting the same set of points. That is, $\gamma_1, \gamma_2 : [0,1] \rightarrow (\mathbb{R}^2-\mathcal{O})$, with $\gamma_i(0) = q_s, \gamma_i(1) = q_e$ are homotopic if and only if $h(\gamma_1) = h(\gamma_2)$.

\begin{figure}
  \begin{center}
  \subfigure[With \emph{rays} as the $U_i$'s: $h(\gamma) = \ldq r_1^{-1} r_4~ r_2^{-1} r_4^{-1} r_6^{-1} \rdq$.]{  \label{fig:rays}
    \hspace{0.01in} \includegraphics[width=0.35\textwidth, trim=0 0 0 0, clip=true]{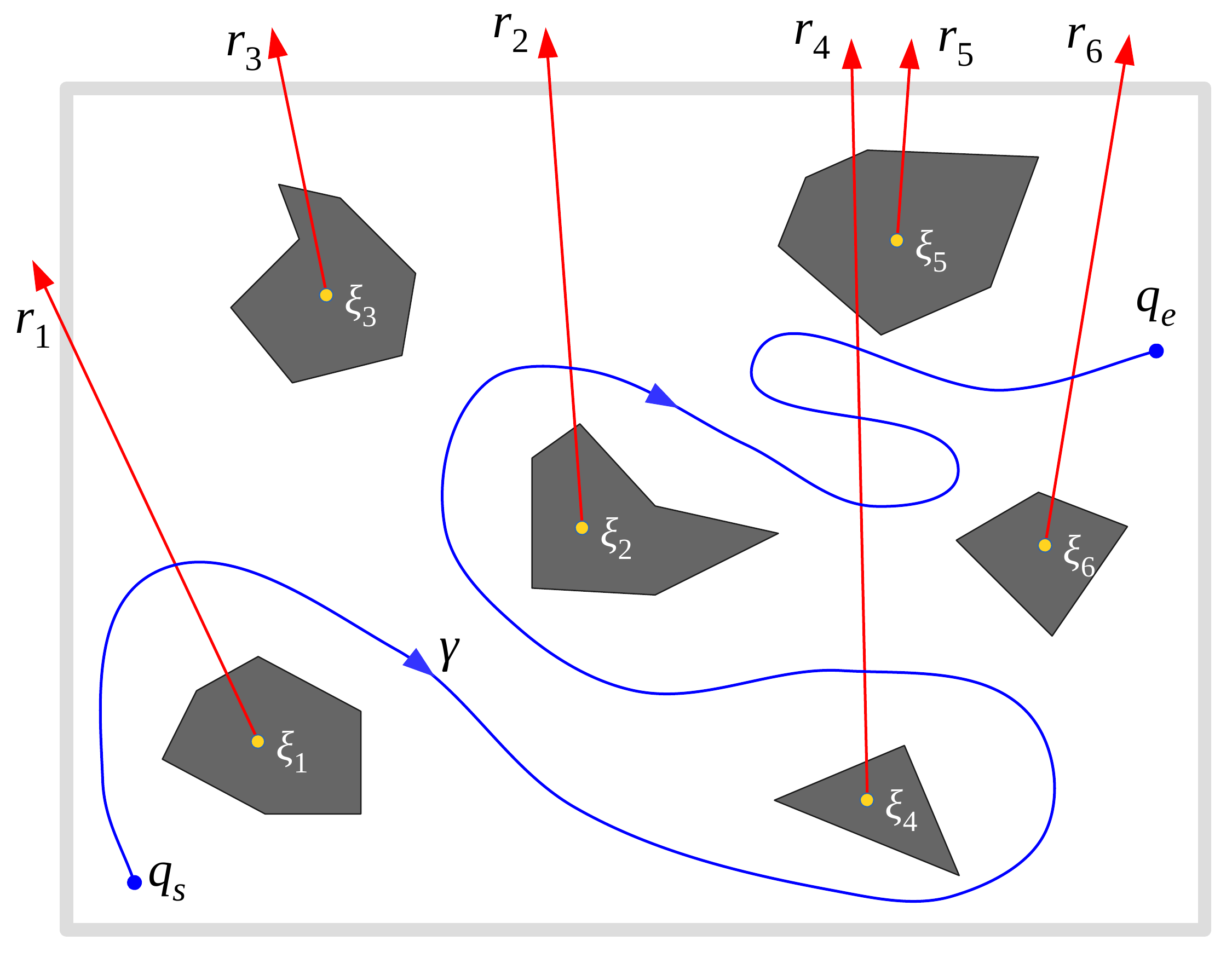} \hspace{0.01in}
    } \hspace{0.1in}
  \subfigure[With a different choice of the $U_i$'s: $h(\gamma) = \ldq u_1^{-1} u_6~ u_6^{-1} u_2~ u_3~ u_5^{-1} \rdq = \ldq u_1^{-1} u_2~ u_3~ u_5^{-1} \rdq $.]{  \label{fig:rays-2}
    \hspace{0.2in} \includegraphics[width=0.35\textwidth, trim=0 0 0 0, clip=true]{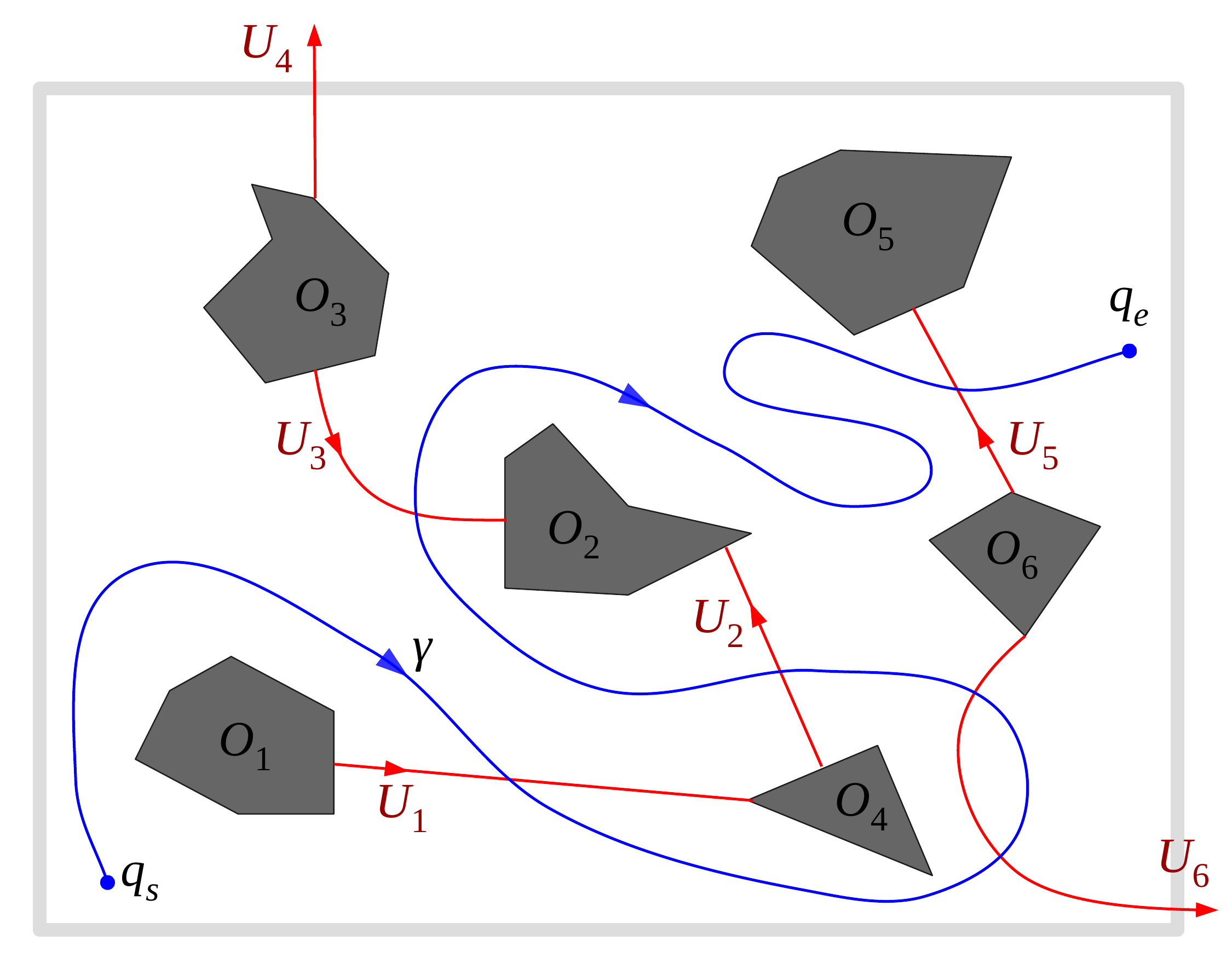} \hspace{0.2in}
    }
  \end{center}
 \caption{Homotopy invariants of curves such as $\gamma$ in $(\mathbb{R}^2-\mathcal{O})$ are words constructed by tracing $\gamma$ and inserting letters in the word for every crossing of the chosen oriented sub-manifolds, $U_i$ (in red).}
\end{figure}


\mysubsection{Words as Homotopy Invariants in Spaces with Free Fundamental Groups}


 In a more general setting the aforesaid construction can be generalized as follows:

\begin{construction} \label{const:free}
Given a $D$-dimensional manifold (possibly with boundary), $X$, suppose $U_1, U_2, \cdots U_n$ are $(D-1)$-dimensional orientable sub-manifolds (not necessarily smooth and possibly with boundaries) such that $\partial U_i \subseteq \partial X$.
 Then, for any curve, $\gamma$ (connecting fixed start and end points, $x_s, x_e\in X$), which is in general position (transverse) w.r.t. the $U_i$'s, one can construct a word by tracing the curve and inserting into the word a letter, $u_i$ or $u_i^{-1}$, whenever
 the curve intersects $U_i$ with a positive or negative orientation respectively.
\end{construction}

The proposition below is a direct consequence of a simple version of the Van Kampen's Theorem (of which several different generalizations are available in the literature).

\begin{proposition} 
\label{prop:van-kampen-simple}
Words constructed as described in Construction~\ref{const:free} are complete homotopy invariants for curves in $X$ joining the given start and end points if the following conditions hold:
 \begin{enumerate}[leftmargin=20pt,labelindent=0pt,itemindent=0pt,labelsep=2pt]
  \item $U_i \cap U_j = \emptyset, ~\forall i\neq j$. \label{item:prop-empty-intersection}
  \item $X- \bigcup_{i=1}^n U_i$ is \changedJ{path connected and} simply-connected, and, \label{item:prop-simply-conn}
  \item $\pi_1(X- \bigcup_{i=1, i\neq j}^n U_i) \simeq \mathbb{Z}, ~\forall j=1,2,\cdots,n$, \label{item:prop-z}
 \end{enumerate}
\end{proposition}

\begin{proof}
 Consider the spaces $X_0 = X- \bigcup_{i=1}^n U_i$ and $X_j = X- \bigcup_{i=1, i\neq j}^n U_i, ~j=1,2,\cdots,n$. Due to the aforesaid properties of the $U_i$'s the set $\mathcal{C}_X = \{X_0, X_1,\cdots,X_n\}$ constitutes an open cover of $X$, is closed under intersection, the pairwise intersections $X_i \cap X_j = X_0, i\neq j$ are simply-connected (and hence path connected), and so are $X_i \cap X_j \cap X_k = X_0, i\neq j \neq k$.

 The proof, when $\gamma$ is a closed loop (\emph{i.e.} $x_s = x_e$), then follows directly from the Seifert-van Kampen theorem~\citep{Hatcher:AlgTop,Crowell:Kampen:59} by observing that $\pi_1(X) \simeq \pi_1(X_0) * \pi_1(X_1) * \pi_1(X_2) * \cdots * \pi_1(X_n) \simeq {*}_{i=1}^n \mathbb{Z}$, the free product of $n$ copies of $\mathbb{Z}$, each $\mathbb{Z}$ generated
due to the restriction of the curve to $X_i, ~i=1,2,\cdots,n$.

 When $\gamma_1$ and $\gamma_2$ are curves (not necessarily closed) joining points $x_s$ and $x_e$, they are in the same homotopy class iff $\gamma_1 \cup -\gamma_2$ is null-homotopic --- that is, $h(\gamma_1) \circ h(-\gamma_2) = \ldq ~ \rdq \Leftrightarrow h(\gamma_1) = h(\gamma_2)$ (where by ``$\circ$'' we mean word concatenation).
\end{proof}

The Construction~\ref{const:free} gives a \emph{presentation}~\citep{epstein1992word} of the fundamental group of $X$ (which, in this case, is a free group due to the Van Kampen's theorem)
as the group
generated by the set of letters $\mymathsf{U} = \{u_1, u_2, \cdots, u_q\}$, and is written as $G = \pi_1(X) = <\! u_1, u_2, \cdots, u_q \!> = <\! \mymathsf{U} \!>$. 
In our earlier construction with the rays, $X=\mathbb{R}^2 - \mathcal{O}$ was the configuration space, and $U_i = X \cap r_i$ were the support of the rays in the configuration space. It is easy to check that the conditions in the above proposition are satisfied with these choices. However such choices of rays is not the only possible construction of the $U_i$'s satisfying the conditions of Proposition~\ref{prop:van-kampen-simple}. Figure~\ref{fig:rays-2} shows a different choice of the $U_i$'s that satisfy all the conditions.



\mysubsection{Simple Extension to $(\mathbb{R}^3-\mathcal{O})$ with Unlinked Unknotted Obstacles}

The construction described in Section~\ref{sec:r2-simple} can be easily extended to the $3$-dimensional Euclidean space punctured by a finite number of un-knotted and un-linked toroidal (possibly of multi-genus) obstacles. 
Instead of ``rays'', in this case the $U_i$'s are $2$-dimensional sub-manifolds
that satisfy the conditions in Proposition~\ref{prop:van-kampen-simple}, with a letter, $u_i$ (or $u_i^{-1}$), being inserted in $h(\gamma)$ every time the curve, $\gamma$, crosses/intersects a $U_i$. This is illustrated in Figures~\ref{fig:3d-simple} and \ref{fig:3d-simple-genus2}.

However, a little investigation makes it obvious that such $2$-dimensional sub-manifolds cannot always be constructed when the obstacle are knotted or linked (Figure~\ref{fig:knot-empty}). One can indeed construct surfaces (\emph{e.g.} Seifert surfaces) satisfying some of the properties, but not all.

\begin{figure}
  \begin{center}
  \subfigure[The surfaces, $U_i$'s, satisfy the conditions of Proposition~\ref{prop:van-kampen-simple}. $h(\gamma) = \ldq u_2~ u_1^{-1} u_2 \rdq$]{  \label{fig:3d-simple}
    \hspace{0.02in}\includegraphics[width=0.23\textwidth, trim=200 140 180 140, clip=true]{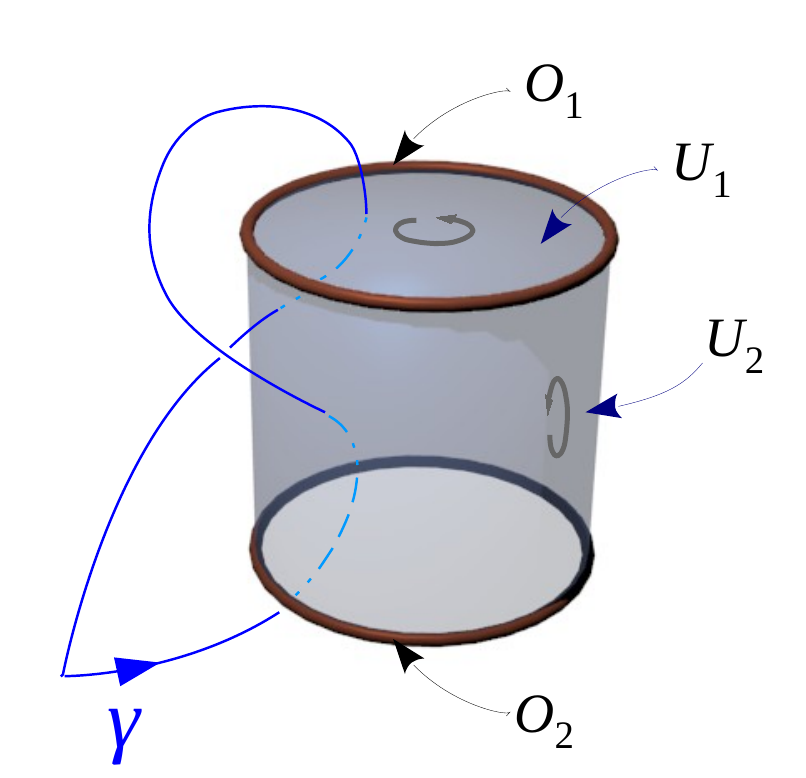} \hspace{0.02in}
    } \hspace{0.01in}
  \subfigure[With a genus-$2$ unknotted obstacle, one can still find, $U_i$'s satisfying the conditions of Proposition~\ref{prop:van-kampen-simple}.]{  \label{fig:3d-simple-genus2}
    \hspace{0.2in}\includegraphics[width=0.23\textwidth, trim=120 100 80 80, clip=true]{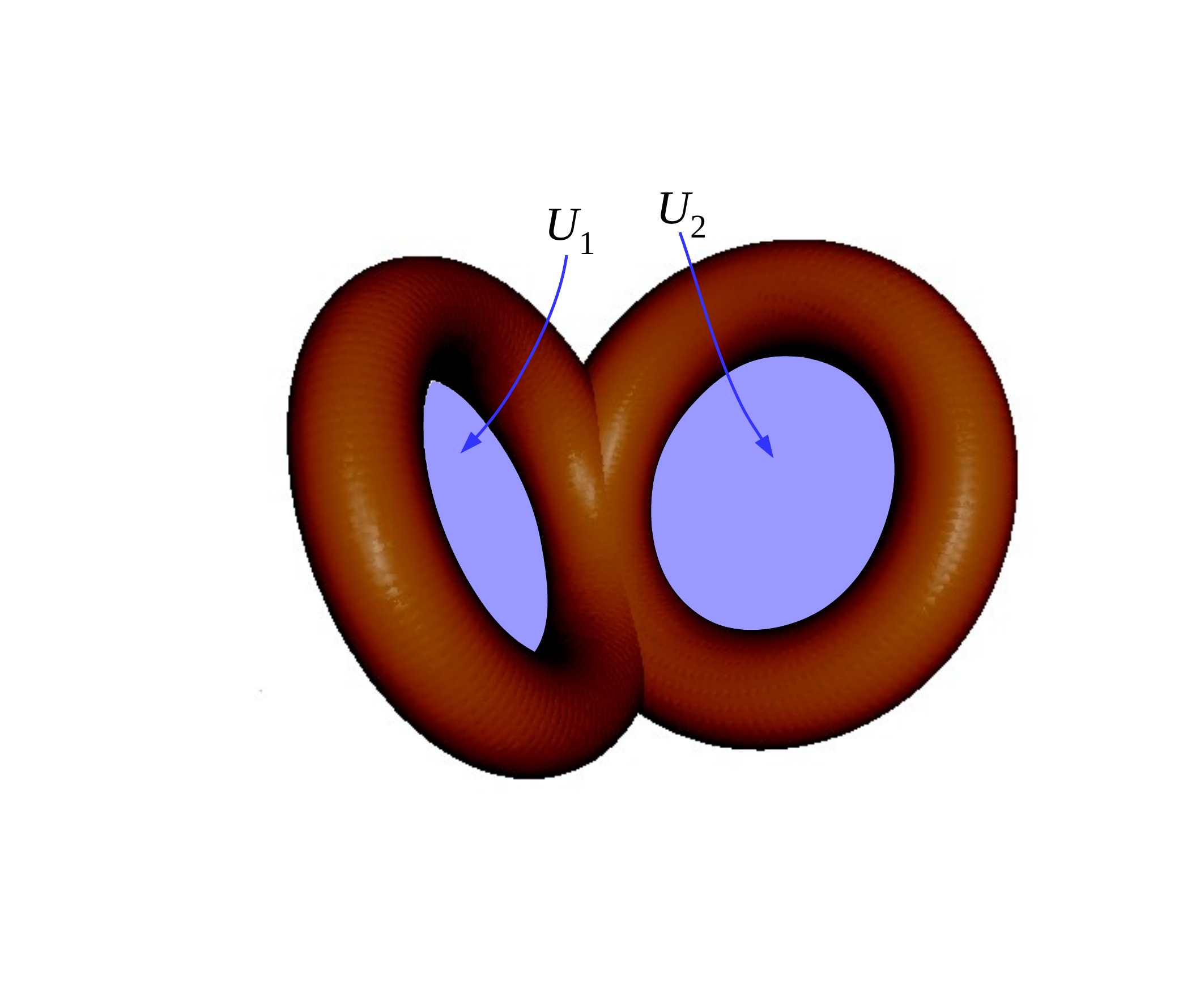}\hspace{0.2in}
    } \hspace{0.01in}
  \subfigure[When the obstacle is a trefoil knot, it's not possible to find $U_i$'s satisfying the conditions of Proposition~\ref{prop:van-kampen-simple}.]{  \label{fig:knot-empty}
   \hspace{0.25in}\includegraphics[width=0.23\textwidth, trim=0 0 0 0, clip=true]{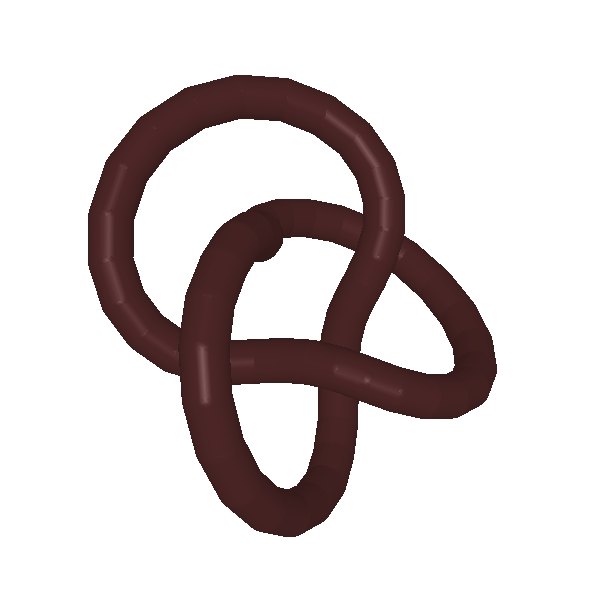} \hspace{0.25in}
    }
  \end{center}
\caption{The fundamental groups of configuration spaces, $\mathbb{R}^3-\mathcal{O}$, may or may not be free.}
\end{figure}


\mysubsection{Application to Graph Search-based Path Planning} \label{sec:graph-search}

Using the homotopy invariants described in the previous sub-section, we describe a graph construction for use in search-based path planning for computing optimal (in the graph) trajectories in different homotopy classes.
\revA{Suppose $G = (V, E)$ is a discrete graph representation of a configuration space, $X$. That is, $V$, is a set of points (vertices) sampled from $X$, and for any two \emph{neighboring} points $x_1,x_2 \in V$ (where neighbors are typically determined by a distance threshold), there exists an edge $(x_1,x_2)\in E$.
\footnote{\revA{We assume that the discretization is such that corresponding to every homotopy class in the configuration space there exists at least one path in the graph that successfully captures/represents that homotopy class. This is true if the sampling of the vertices is sufficiently dense in $X$, and the edge connecting neighboring vertices is the shortest path (with respect to the metric of the underlying space) connecting those points in $X$.}}
We assume $x_s\in V$.}

We first fix the set of sub-manifolds $\{U_1, U_2, \cdots, U_n\}$ as described earlier.
Now, given a discrete graph representation of the configuration space, $G = (V, E)$,
we construct an $h$-augmented graph, $G_h = (V_h, E_h)$, which is essentially a \emph{lift} of $G$ into the \emph{universal covering space} of $X$~\citep{Hatcher:AlgTop}. The construction of such augmented graphs has been described in our prior work~\citep{cable:separation:IJRR:14,ICRA:14:tethered,DARS:14:HRI}, and the explicit construction of $G_h$ can be described as follows:
\begin{enumerate} [leftmargin=15pt,labelindent=0pt,itemindent=0pt,labelsep=2pt,label=\roman*.]
 \item Vertices in $V_h$ are tuples of the form $(x, w)$, where $x\in V$ and $w$ is a word made out of letters $u_i$ and $u_i^{-1}$. \label{item:Gh-vertex-desc}
 \item $(x_s, \ldq~\rdq) \in V_h$. \label{item:Gh-first-vertex}
 \item For every edge $[x_1,x_2]\in E$ and every vertex $(x_1,w)\in V_h$, there exists an edge \revB{$[(x_1,w), (x_2,w\circ h(\overrightarrow{x_1 x_2}))] \in E_h$}, where $\overrightarrow{x_1 x_2}$ denotes the directed curve that constitutes the edge $[x_1,x_2]$. \label{item:incr-const}
 \item The length/cost of an edge in $G_h$ is same as its projection in $G$: $C_{G_h}\left([(x_1,w_1), (x_2,w_2)]\right) = C_G([x_1,x_2])$.
\end{enumerate}

The item `\ref{item:Gh-vertex-desc}' is just a qualitative description of the vertices in $G_h$. Item `\ref{item:Gh-first-vertex}' describes one particular vertex in $G_h$, and using that, item `\ref{item:incr-const}' describes an incremental construction of the entire graph $G_h$.
The topology of $G_h$ can be described as a lift of $G$ into the universal covering space, $\widetilde{X}$, of $X$, and is illustrated in Figure~\ref{fig:cover} for a uniform cylindrically discretized space with a single disk-shaped obstacle.

\begin{figure}
   \begin{center}
  \subfigure[The configuration space, $X$ (light gray), and the vertex set, $V$ (blue dots).]{  \label{fig:cover-G}
   \hspace{0.35in}\includegraphics[height=0.25\textwidth, trim=0 20 0 0, clip=true]{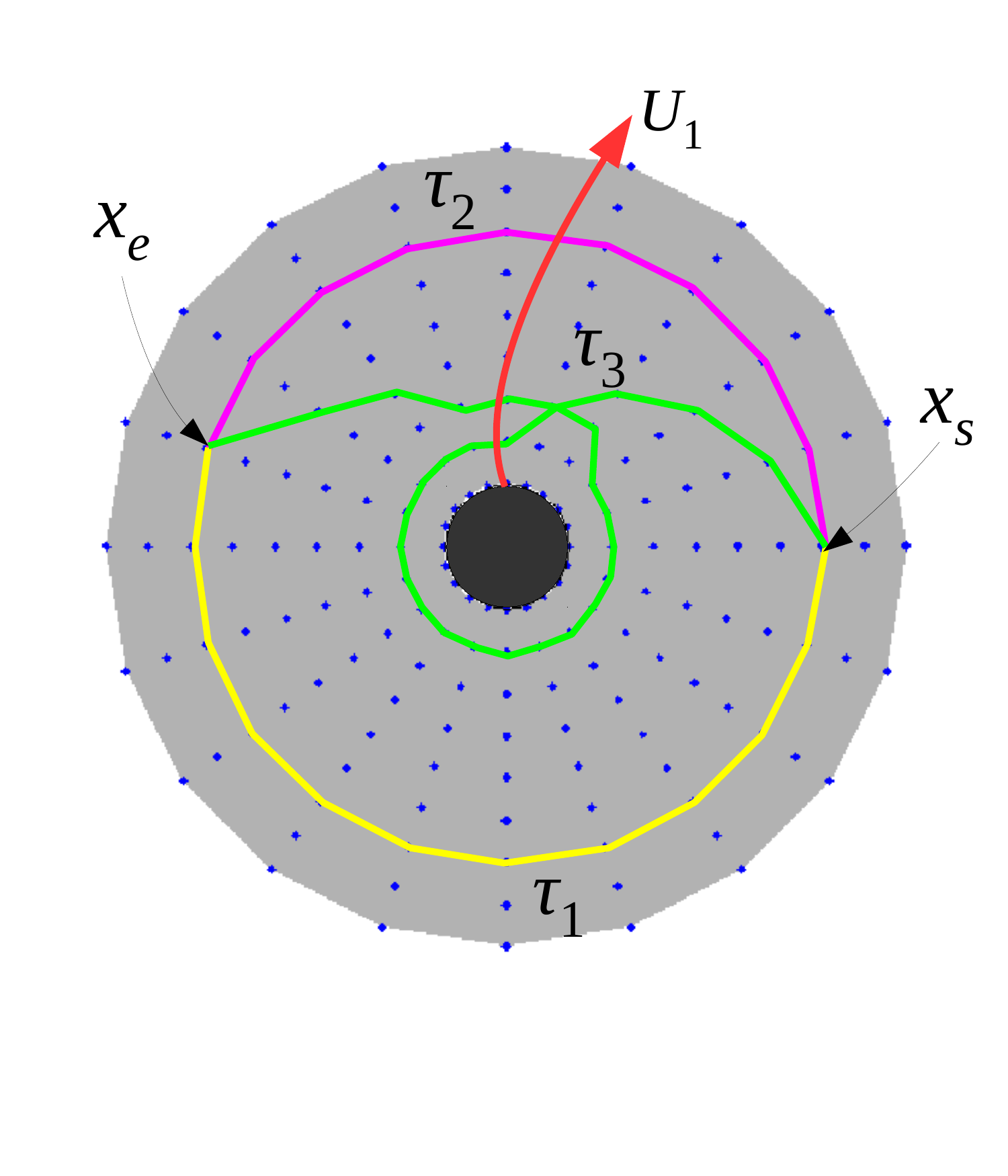}\hspace{0.35in}
    } \hspace{0.05in}
    \subfigure[The universal covering space, $\widetilde{X}$, and the vertex set, $V_h$. Note how the trajectories lift to have different end points.]{  \label{fig:cover-Gh}
    \hspace{0.55in}\includegraphics[height=0.25\textwidth, trim=0 0 0 0, clip=true]{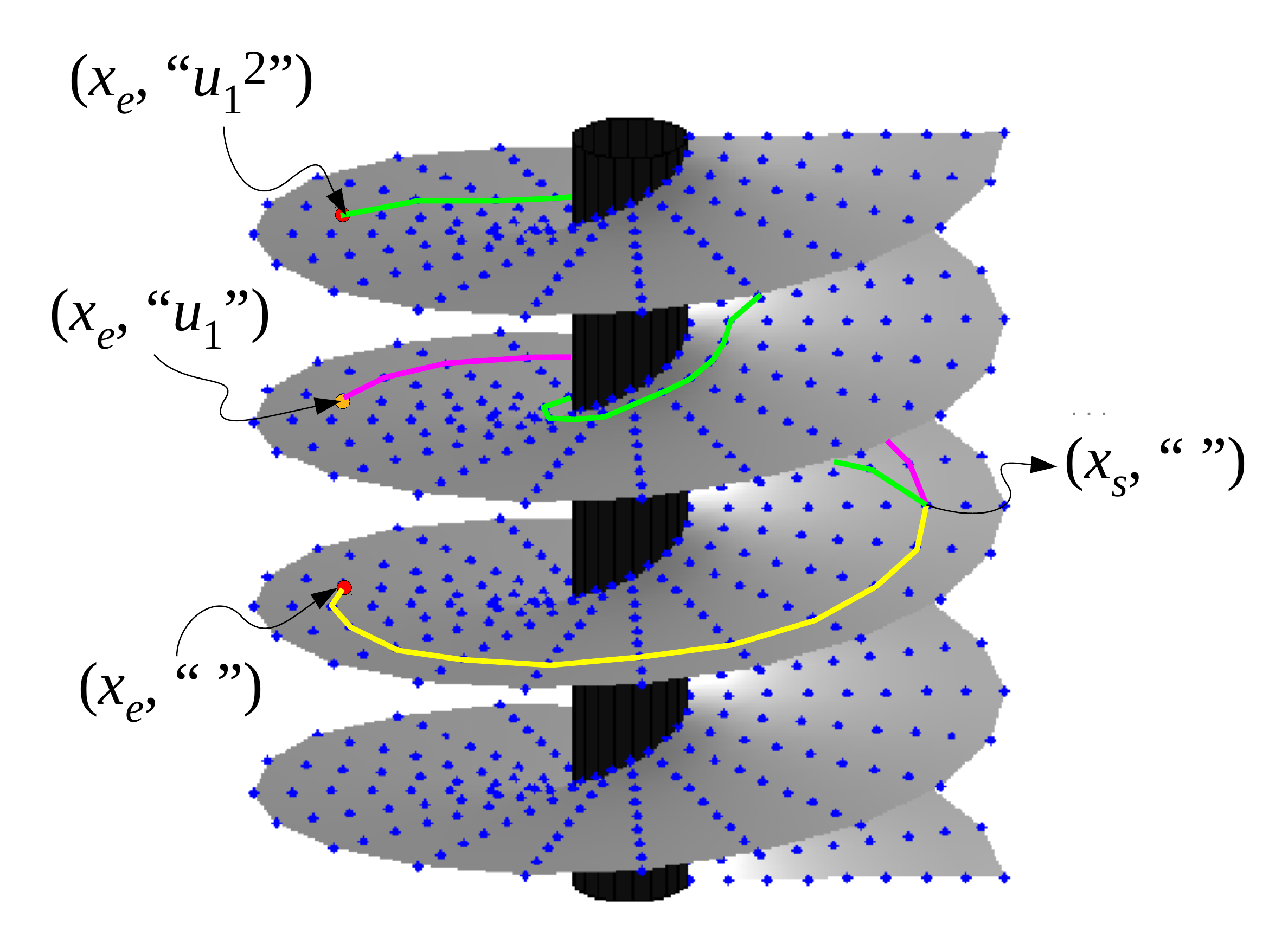}\hspace{0.55in}
    }
  \end{center}
 \caption{The $h$-augmented graph, $G_h$, is a lift of $G$ into $\widetilde{X}$.}
\label{fig:cover}
\end{figure}

Such an incremental construction is well-suited for use in graph search algorithms such as Dijkstra's or A*~\citep{cormen2001}, in which one initiates an \emph{open set} using the start vertex (in item \ref{item:Gh-first-vertex}), and then gradually \emph{expands} vertices, generating only the neighbors at every expansion (the recipe for which is given by item `\ref{item:incr-const}').
Executing a search (Dijkstra's/A*) in $G_h$ from $(x_s, \ldq ~\rdq)$ to vertices of the form $(x_e, *)$ (where `$*$' denotes any word), and projecting it back to $G$, gives us optimal trajectories in $G$ that belong to different homotopy classes.
Figure~\ref{fig:result-free-2d} shows $5$ such optimal trajectories in the graph, connecting a given start and goal vertex, where $G$ was constructed by an uniform hexagonal discretization of the planar configuration space. 
One can then employ a simple curve shortening algorithm~\citep{ICRA:14:tethered} to obtain 
ones more optimal than the
ones restricted to $G$ (Figure~\ref{fig:result-free-2d-shortened}).
Similarly, shortest trajectories connecting $x_s$ and $x_e$ can be obtained in $3$-dimensional configuration spaces with free fundamental group (\emph{e.g.}, Figure~\ref{fig:result-free-3d}, \revA{showing $4$ paths connecting two fixed points in $\mathbb{R}^3$ with two un-linked toroidal obstacles}).

\begin{figure}
   \begin{center}
  \hspace{-0.2in}
  \subfigure[$5$ shortest trajectories in $G$ belonging to different homotopy classes. (obstacles in gray)]{  \label{fig:result-free-2d}
   \hspace{0.0in}\fbox{\includegraphics[height=0.21\textwidth, trim=0 20 0 0, clip=true]{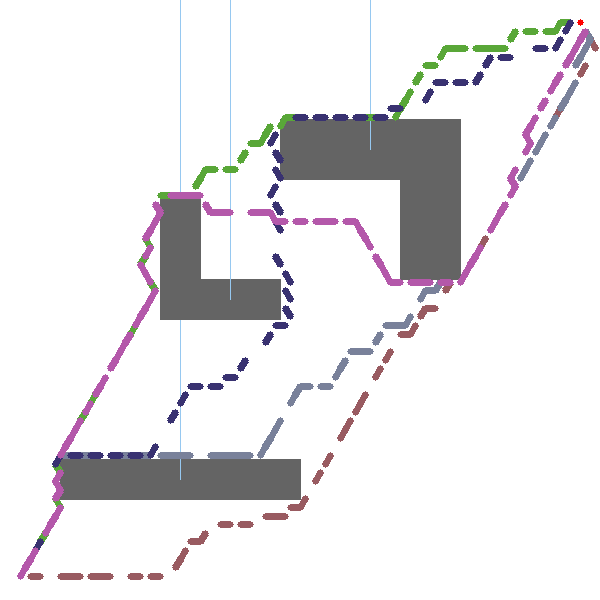}}\hspace{0.0in}
    } 
  \subfigure[The trajectories after being \emph{shortened}, but belonging to same homotopy classes.
  ]{  \label{fig:result-free-2d-shortened}
   \hspace{0.0in}\fbox{\includegraphics[height=0.21\textwidth, trim=0 20 0 0, clip=true]{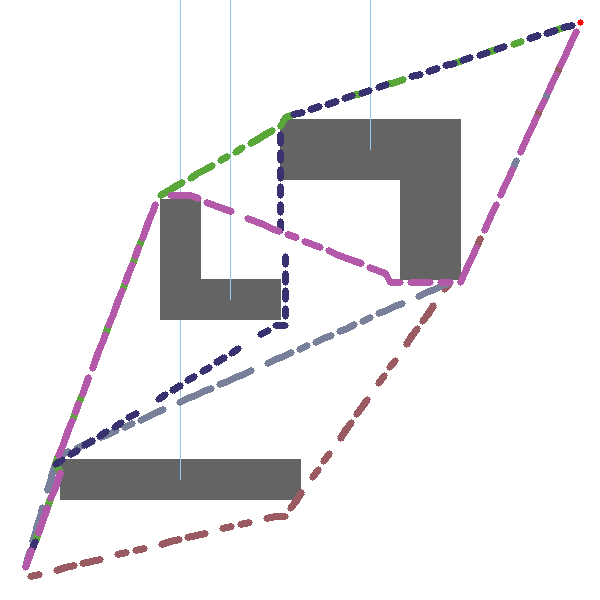}}\hspace{0.0in}
    } 
  \subfigure[By setting $x_s=x_e=x_0$, the same method can be used to find shortest loops passing through $x_0$.]{  \label{fig:result-free-2d-loops}
   \hspace{0.0in}\fbox{\includegraphics[height=0.21\textwidth, trim=0 20 0 0, clip=true]{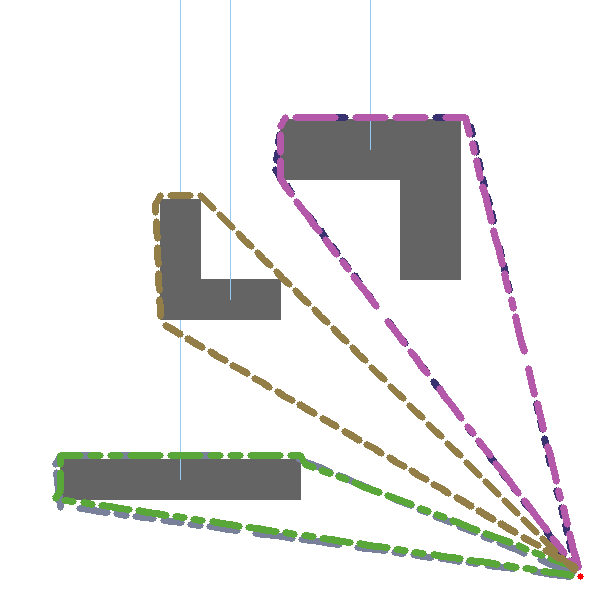}}\hspace{0.0in}
    } 
  \subfigure[Similar computation in $\mathbb{R}^3-\mathcal{O}$, when its fundamental group is free.]{  \label{fig:result-free-3d}
   \hspace{0.0in}\fbox{\includegraphics[height=0.21\textwidth, trim=0 20 0 0, clip=true]{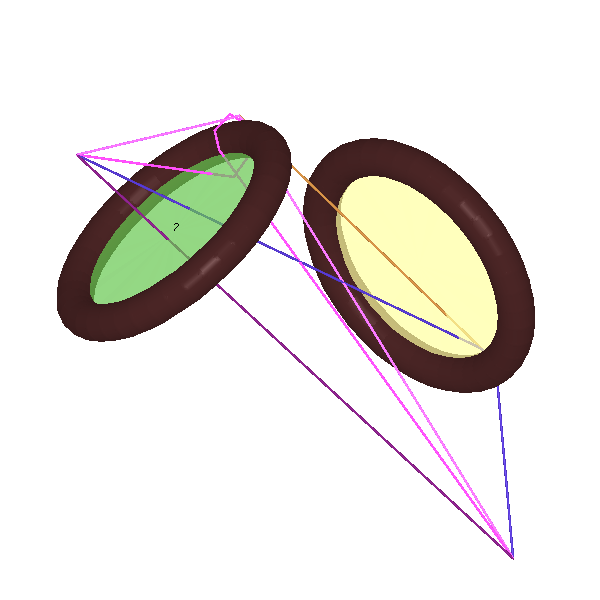}}\hspace{0.0in}
    }
    \hspace{-0.3in}
  \end{center}
\caption{Simple results in configuration spaces that have free fundamental groups. The dot/dash pattern and colors are shown to distinguishing between the trajectories.} \label{fig:result-free}
\end{figure}

%
%


\mysection{Knot and Link Complements}


As described earlier, when the obstacle set in $\mathbb{R}^3$ consists of knots and links, it is in general not possible to find the sub-manifolds $U_i \subset (\mathbb{R}^3 - \mathcal{O})$ as required by Proposition~\ref{prop:van-kampen-simple}.
However, thankfully we have more generalized versions of the Van Kampen theorem at our disposal that lets us extend the proposed methodology to such spaces. We first illustrate the generalization in $\mathbb{R}^3 - \mathcal{O}$ using knot/link diagrams.

\mysubsection{Dehn Presentation of Fundamental Group of Knot/Link Complements}

For simplicity we consider
knots and links 
in $\mathbb{R}^3$ as obstacles.
We assume that the knots/links are described by polygons, all of which together constitute $O\subset \mathbb{R}^3$.
The \emph{thickened} obstacles (the knots/links with the tubular neighborhoods) will be referred to as $\mathcal{O}$.
We consider a knot/link diagram~\citep{lickorish1997introduction} of the obstacles:
Given a projection map, $p: \mathbb{R}^3 \rightarrow \mathbb{R}^2$, the knot/link diagram is the projection of the knot/link, $p(O)$, along with additional information about the \emph{$z$-ordering} at the self-intersections of $p(O)$.
We assume that in this diagram the self-intersections are all transverse (which can always be achieved through infinitesimal perturbations) and that the diagram divides the plane into simply-connected regions (say $q$ counts of them) each bounded by segments of the projected obstacles, and one unbounded exterior region.
%
The \emph{boundary} (the boundary of the closure) of each of the bounded regions is itself a polygon, $Q_i \subseteq p(O), ~i=1,2,\cdots,q$ 
(Figure~\ref{fig:knot-diagram}). 
Clearly $p^{-1}(Q_i) \cap O$ (the preimage of $Q_i$ in the original obstacle) will be a discontinuous polygon, with discontinuities at the preimages of the self-intersection points on the knot diagram. But these discontinuities can be removed simply by ``connecting'' the preimages at each self-intersection point, resulting into a spatial polygon, $\widetilde{Q}_i$ with the property that $p(\widetilde{Q}_i) = Q_i$.
A simple triangulation can then be employed to construct a surface, $U_i$, in $\mathbb{R}^3-O$, such that its boundary is $\widetilde{Q}_i$ and $p(U_i)$ is the simply-connected region bounded by $Q_i$ (this can be achieved by first triangulating the planar region, $p(U_i)$, and then \emph{lifting} the triangulation to $\mathbb{R}^3$) -- see Figure~\ref{fig:knot-diagram}.
%

The $U_i$'s thus constructed satisfy properties (\ref{item:prop-simply-conn}) and (\ref{item:prop-z}) of Proposition~\ref{prop:van-kampen-simple}, but not property (\ref{item:prop-empty-intersection}), nor do they satisfy the property $\partial U_i \subseteq \partial X$. 
The consequence of this is that near the regions where the $U_i$'s intersect, there can be closed loops in $\mathbb{R}^3 - O$ which are null-homotopic, but words constructed simply by tracing the loop and inserting letters corresponding to intersections with the $U_i$'s, as we did earlier, may not be the empty word (identity element). An example is illustrated in Figure~\ref{fig:knot-surfs}).
Due to our construction, such intersection of the $U_i$'s happen only along lines passing through the pre-image of the self-intersections in the knot diagram, for each of which we end up getting a null-homotopic closed loop with non-empty word.


\begin{figure}
   \begin{center}
  \subfigure[Knot diagram, showing one of the polygons, $Q_3$ (in cyan).]{ \label{fig:knot-diagram}
   \hspace{0.22in}\includegraphics[height=0.22\textwidth, trim=50 50 50 50, clip=true]{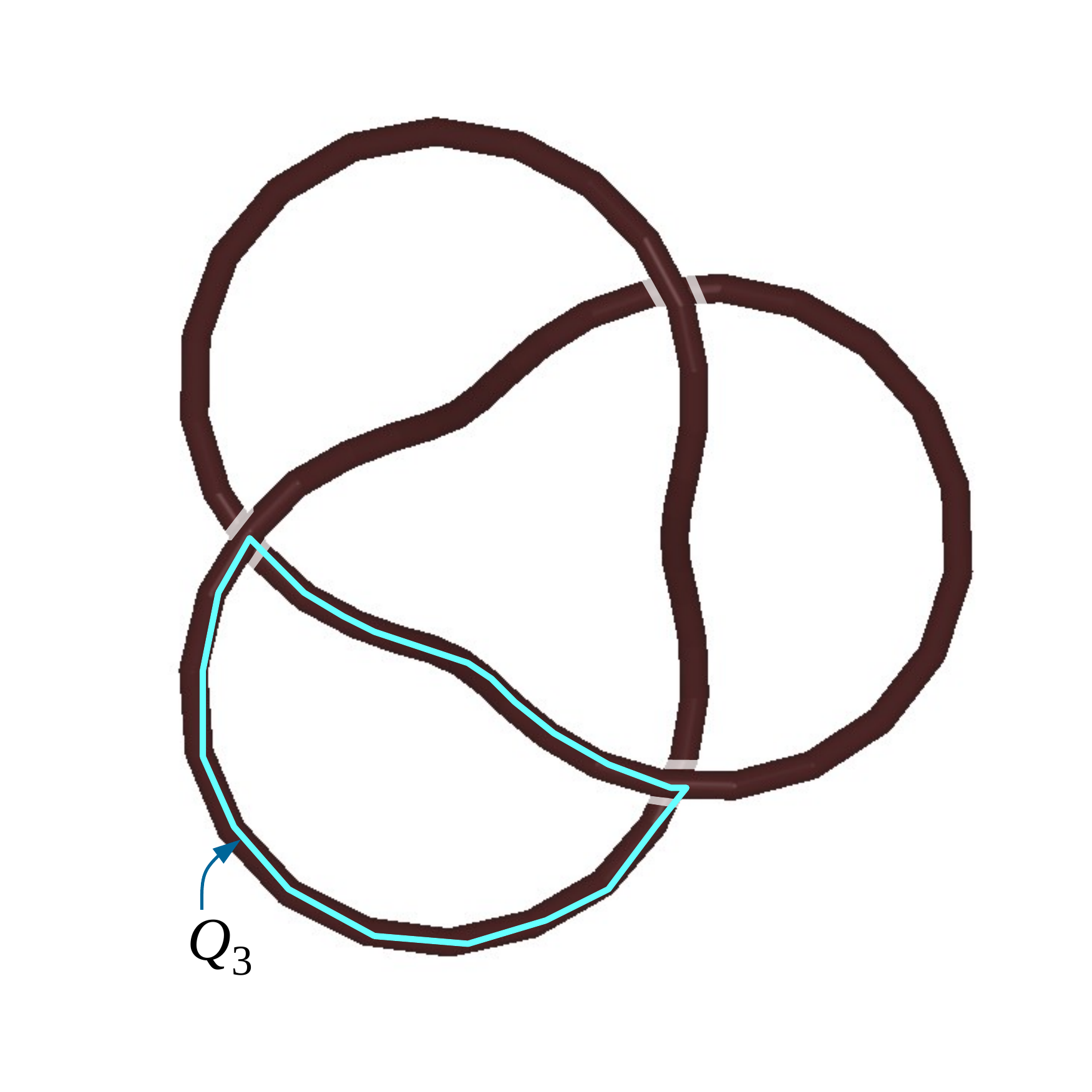}\hspace{0.22in}
    } \hspace{0.03in}
  \subfigure[The surfaces, $U_i$, shown in different colors. The closed loop $\gamma$ is null-homotopuc, but $h(\gamma)=\ldq u_1^{-1} u_2~ u_3^{-1} \rdq$. Top and side views.]{ \label{fig:knot-surfs}
   \hspace{0.2in}\includegraphics[height=0.22\textwidth, trim=50 50 50 50, clip=true]{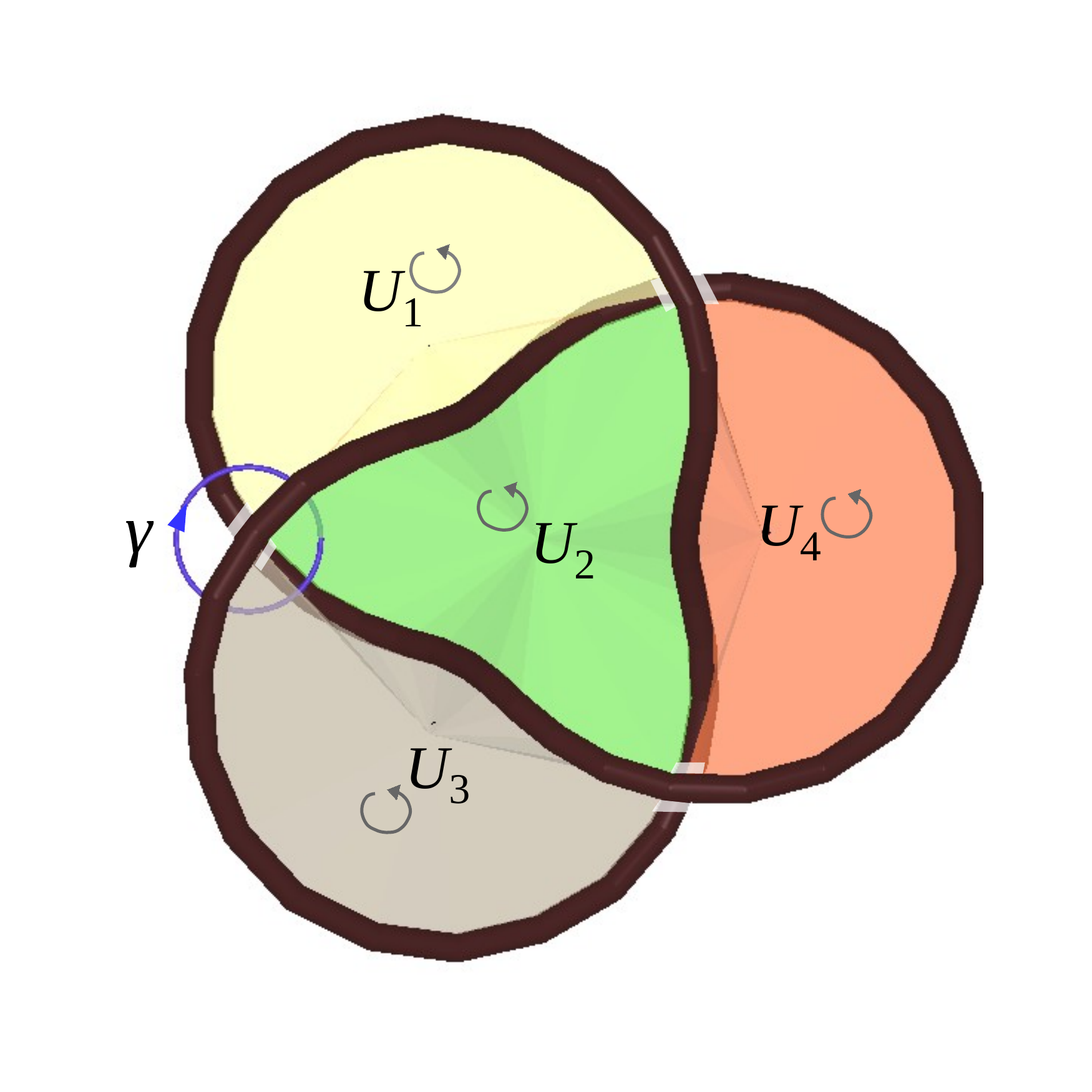}\hspace{0.2in}
   \hspace{0.2in}\includegraphics[height=0.22\textwidth, trim=0 10 0 50, clip=true]{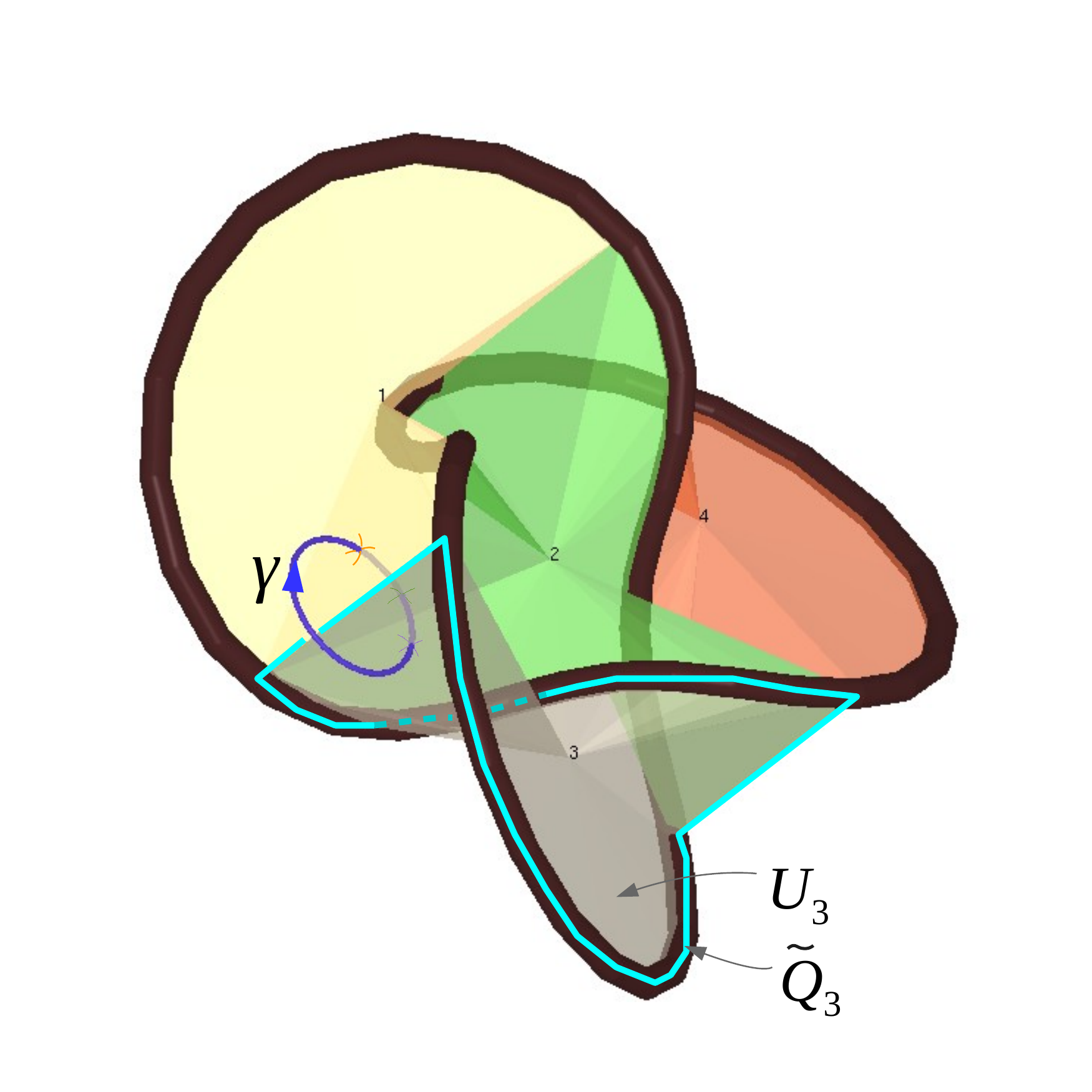}\hspace{0.2in}
    } \hspace{0.02in}
  \end{center}
\caption{Constructing the surfaces, $U_i$, from polygonal knot/link diagrams (polygon segments shown as thickened cylinders for easy visualization). Null-homotopic loops as $\gamma$ have non-empty words.} \label{fig:knot-surf}
\end{figure}

The Dehn presentation~\citep{Weinbaum:knot:word:71} uses surfaces as constructed to describe the fundamental group of knot/link complements. 
We consider the free group, 
$G = <\! u_1, u_2, \cdots, u_q \!>$ $= <\! \mymathsf{U} \!>$.
In general, for every self-intersection in the knot/link diagram, there are four adjacent surfaces, $U_{i_1}, U_{i_2}, U_{i_3}$ and $U_{i_4}$ in the order as shown in Figure~\ref{fig:intersection} (when the self-intersection is adjacent to the unbounded region in the knot diagram, there are only three). Correspondingly, there is a closed null-homotopic loop, $\gamma_i$, that has a word $\rho_i = \ldq u_{i_1}\, u_{i_2}^{-1} u_{i_3}\, u_{i_4}^{-1} \rdq$. Thus we have such words $\rho_1, \rho_2, \cdots, \rho_m$ (assuming there are $m$ counts of self-intersections) that represent null-homotopic loops. These words are called \emph{relations} and we call the set $\mymathsf{R} = \{\rho_1, \rho_2, \cdots, \rho_m\}$ the \emph{relation set}.
It can be easily noted that inverses 
and cyclic permutations
of each $\rho_i$ also corresponds to null-homotopic loops. 
We thus define 
the \emph{symmetricized relation set}, $\overline{\mymathsf{R}}$, as the set containing all the words in $\mymathsf{R}$, all their inverses, and all cyclic permutation of each of those.

Let the normal subgroup of $G$ generated by $\overline{\mymathsf{R}}$ be $N = \{\ldq \alpha_1 \rho_{i_1} \alpha_1^{-1} ~ \alpha_2 \rho_{i_2} \alpha_2^{-1} \cdots \newline \alpha_\kappa^{-1} ~ \alpha_\kappa \rho_{i_\kappa} \cdots \rdq ~|~ \alpha_k \!\in\! G,~ \rho_{i_k} \!\in\! \overline{\mymathsf{R}} \} = <\! \overline{\mymathsf{R}}^G \!>$ (normal closure 
of $\overline{\mymathsf{R}}$ in $G$).
It is easy to observe that a closed loop, $\gamma$, in $X = \mathbb{R}^3-\mathcal{O}$, has a word that is an element of $N$ iff it is null-homotopic.
Due to a more general version of the Van Kampen's theorem~\citep{Hatcher:AlgTop}, the fundamental group of $X$ is the quotient group, $\pi_1(X) = G/N = <\! \mymathsf{U} ~|~ \mymathsf{R} \!>$ --- the group in which, under the quotient map, elements of $N$ are mapped to the identity element.
\changedJ{This is summarized and generalized in the following proposition.}

\begin{figure}
   \begin{center}
   \includegraphics[height=0.25\textwidth, trim=0 0 0 0, clip=true]{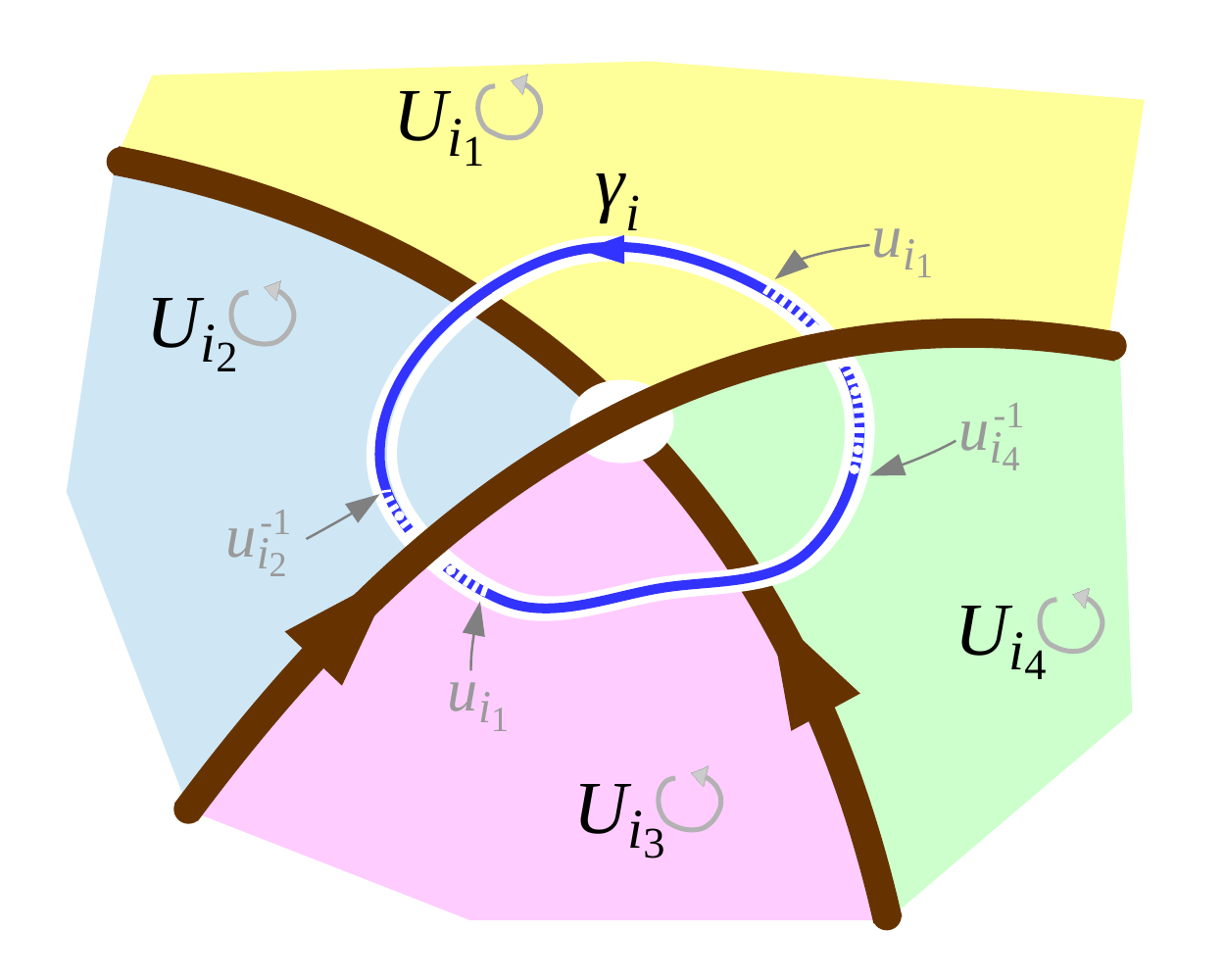}
  \end{center}
 \caption{A ``self-intersection'' in a knot/link diagram, with a null-homotopic loop, $\gamma_i$, intersecting the surfaces adjacent to the intersection. $h(\gamma_i) = \ldq u_{i_1}\, u_{i_2}^{-1} u_{i_3}\, u_{i_4}^{-1} \rdq$. One or more of the surfaces can be non-existent, in which case the corresponding letters are simply absent from the word. These words constitute $\mymathsf{R}$, and should map to the identity element in $\pi_1(X) = <\! \mymathsf{U} ~|~ \mymathsf{R} \!>$} \label{fig:intersection} 
\end{figure}

\changedJ{
\begin{proposition} 
\label{prop:van-kampen-general}
Given a $D$-dimensional manifold (possibly with boundary), $X$, suppose $U_1, U_2, \cdots U_n$ are $(D-1)$-dimensional orientable sub-manifolds (not necessarily smooth and possibly with boundaries) such that $\partial U_i \subseteq \partial X \cup \bigcup_{j\neq i} U_j $ and
 \begin{enumerate}[leftmargin=20pt,labelindent=0pt,itemindent=0pt,labelsep=2pt]
  \item[1.]a) The intersections of $U_i$ and $U_j$, if non-empty, are transverse (hence $(D-2)$-dimensional), \label{item:prop2-nonempty-intersection}
  \item[~~]b) $\bigcap_{i=1}^n U_i = \emptyset$, \label{item:prop2-empty-joint-intersection}
  \item[2.] $X- \bigcup_{i=1}^n U_i$ is \changedJ{path connected and} simply-connected, and, \label{item:prop2-simply-conn}
  \item[3.] $\pi_1(X- \bigcup_{i=1, i\neq j}^n U_i) \simeq \mathbb{Z}, ~\forall j=1,2,\cdots,n$, \label{item:prop2-z}
 \end{enumerate}
 For a loop in $X$, as before, one can construct a word as described in Construction~\ref{const:free}. Let the set of all the letters constituting such words be $\mymathsf{U}$.
 
 For every $(D-2)$-dimensional intersection of the $U_i \cap U_j, ~i\neq j$, one can construct a loop in the tubular neighborhood of $U_i \cap U_j$ in $X$ that links with $U_i \cap U_j$. Let the set of words corresponding to such loops be $\mymathsf{R}$.
 
 Then the fundamental group of $X$ is isomorphic to $<\! \mymathsf{U} ~|~ \mymathsf{R} \!>$.
\end{proposition}
}

\mysubsection{The Word Problem and Dehn Algorithm}

Due to the discussion above, 
two trajectories, $\gamma_1, \gamma_2$, connecting $x_s$ and $x_e$ in the knot/link complement, $X$, belong to the same homotopy class iff
the word $h(\gamma_1 \cup -\gamma_2) = h(\gamma_1) \circ h(\gamma_2)^{-1}$ belongs to
$N  = <\! \overline{\mymathsf{R}}^G \!>$. 
This problem in group theory is known as the \emph{word problem}~\citep{epstein1992word}, and there are various algorithms, each suitable for specific types of groups, for solving the word problem. We, in particular, will focus on a very simple algorithm due to Max Dehn~\citep{lyndon2001combinatorial,greendlinger:dehn:1986}, which is applicable to a wide class of groups and their presentations.

\emph{Dehn's metric algorithm:}
Given a presentation of a group, $\pi_1 = <\! \mymathsf{U} ~|~ \mymathsf{R} \!>$, we construct the \emph{symmetricized} relation set $\overline{\mymathsf{R}}$ 
as described earlier.
%
Given a cyclically reduced word, $w$, made up of letters (and their inverses) from $\mymathsf{U}$, one checks for every element $\rho \in \overline{\mymathsf{R}}$ if $w$ and $\rho$ share a common sub-words that is of length greater than $|\rho|/2$ ($|\rho|$ being the length of $\rho$). If they do (say, $\rho = \alpha\beta\gamma$, with $\beta$ being a sub-word appearing in $w$, and $|\beta|>|\rho|/2$), we replace the sub-word with the shorter equivalent that one obtains by setting $\rho$ to the identity element (\emph{i.e.}, replace $\beta$ by $\alpha^{-1}\gamma^{-1}$ in $w$).
This process is repeated, and the algorithm terminates when no more such sub-words are found.
The final word at which the algorithm terminates indicates if the initial word, $w$, is in $N$ (whether it maps to the identity element in $\pi_1$).

\changed{This algorithm can be used in conjunction with search in $G_h$ as before for finding optimal trajectories in different homotopy classes, with two vertices $(x,w_1), (x,w_2) \in V_h$ being the same iff $h(w_1)\circ h(-w_2)$ reduces to the empty word upon applying the Dehn's metric algorithm.}

\mysubsection{Guarantees of Dehn Algorithm}

It's well known \citep{greendlinger:dehn:1986} that if Dehn algorithm terminates at the empty word, then $w\in N$.
However, the converse is not necessarily true. One can derive several sufficient (and often highly restrictive) conditions on the presentation $<\! \mymathsf{U} ~|~ \mymathsf{R} \!>$ under which the converse holds \citep{lyndon2001combinatorial}.
If, for a given presentation of a group
the converse holds,
we say that Dehn algorithm is \emph{complete} for that presentation (or that the presentation is complete with respect to the Dehn algorithm, or that the word problem is solvable using the specific presentation and Dehn algorithm).

Due to the result of \citep{Weinbaum:knot:word:71}, the Dehn presentation of the fundamental groups of the complement of a tame, alternating, prime knot is complete with respect to the Dehn algorithm. 
%
%
\changed{It is also known~\citep{epstein1992word} that automatic groups (including hyperbolic groups) have presentations that are complete with respect to Dehn algorithm.}



\mysection{Cylindrically-deleted Configuration Spaces}

The previous results are limited to 3-dimensional spaces: one suspects that higher dimensions are more difficult.  However, there are some classes of spaces for which optimal path-planning with homotopy constraints is still computable via a Dehn algorithm, independent of dimension.  The following class of examples is inspired by robot coordination problems, in which individual agents with predetermined motion paths have to coordinate their motions so as to avoid collision. 

Consider a collection of $n$ graphs $(\Gamma_i)_1^n$, each embedded in a common workspace (usually $\mathbb{R}^2$ or $\mathbb{R}^3$) with intersections permitted. In the simplest case, each $\Gamma_i$ will be homeomorphic to a closed interval, but more general graphs are permitted, such as roadmap approximations to a configuration space. On each $\Gamma_i$, a robot $R_i$ with some particular fixed size/shape is free to move. Such motion may be Euclidean (by translation/rotation); more general motions are possible, so long as the region occupied by the robot $R_i$ in the workspace is purely a function of location on $\Gamma_i$. A point in the product space $\prod_i\Gamma_i$ determines the locations of the $n$ robots in the common workspace.
Certain configurations are illegal, due to collisions. For example, if the robots are point-like, and each $\Gamma_i=\Gamma$ is identical, then the configuration space of $n$ points on $\Gamma$ is the cross product $\prod_i\Gamma_i$ minus the pairwise diagonal $\Delta$.  If the robots are given finite extent, then this system has a configuration space obtained by the graph product $\prod_i\Gamma_i$ minus an $\epsilon$-neighborhood of the pairwise diagonal. However, more general types of collisions can be defined, say, if the robots are irregularly shaped and the graphs $\Gamma_i$ are all different. In this most general case, the natural analogue of a configuration space is the coordination spaces of \citep{GL:2006}.

The {\em coordination space} of this system is defined to be the space of all configurations in $\prod_i\Gamma_i$ for which there are no collisions -- the geometric robots $R_i$ have no overlaps in the workspace. Under the assumption that collisions between robots are pairwise-defined, the coordination space is {\em cylindrically deleted} and of the form 

{\small \begin{equation*}
X = 
\left(\prod_{i=1}^n \Gamma_i\right) - {\mathcal O} 
\quad {\rm where } \quad
{\mathcal O} = \bigcup_{i<j}\left\{
  (x_k)_1^N\in\prod_{k=1}^{N} \Gamma_k 
  \, : \, (x_i,x_j)\in\Delta_{i,j}
  \right\},
\end{equation*}}

\noindent for some (open, ``collision'') sets $\Delta_{i,j}\subset\Gamma_i\times\Gamma_j$ where $1\leq i<j\leq N$. In what follows, we assume that the $\Delta_{i,j}$ are sufficiently tame (e.g., semialgebraic) so as to avoid issues of non-finitely-generated $\pi_1$. Given (internal, path-) metrics on each $\Gamma_i$, the coordination space $X$ inherits a locally-Euclidean metric on products of edges in the graphs.  Such $X$ are complete path-spaces and thus the problem of geodesics is well-posed. Their fundamental groups can be (highly) nontrivial, depending on the obstacle set ${\mathcal O}$. However, finding optimal paths subject to homotopy classes is still computable.
\changed{To that end one can construct the subspaces $U_i \subset X$ of co-dimension $1$, and the relation set $\mymathsf{R}$, 
and use them to design complete homotopy invariants as before.
We do not discuss the explicit construction of the $U_i$'s for cylindrically-deleted coordination spaces in this paper, but provide the following theorem on solvability of the word problem in such spaces.}

\begin{theorem}
\label{thm:coordDehn}
Any compact cylindrically-deleted coordination space $X$ admits a Dehn algorithm for $\pi_1$. 
\end{theorem}

\begin{proof}
Any such $X$ is realized as a Hausdorff limit of cubcial complexes which were shown in \cite[Thm 4.4]{GL:2006} to be nonpositively-curved and to stabilize in $\pi_1$ by tameness. All nonpositively-curved piecewise-Euclidean cube complexes have fundamental groups which are, by a famous result of Niblo-Reeves \citep{NRgeometry98}, {\em biautomatic}. Biautomatic groups all admit a Dehn algorithm (specifically, there is a quadratic isoperimetric inequality) \citep{epstein1992word}. 
\end{proof}

It is worth noting that $\ell_2$-shortest paths 
are perhaps not the most natural optimization for coordination spaces. It would be interesting to consider other ($\ell_1$, $\ell_\infty$) pointwise norms.

\changedJ{
\vspace{0.1in}
In the next section we consider point robots navigating on a plane.
The configuration space of each robot is the Euclidean plane. For such a configuration space for the individual robots, the result of Theorem~\ref{thm:coordDehn} may not hold, since the fundamental group of the coordination space may not be biautomatic. Nevertheless, we can construct a presentation of the fundamental group, which we do, and still apply Dehn metric algorithm to it, although the Dehn algorithm may not be complete for the proposed presentation.

\mysubsection{Presentation of the Fundamental Group of a Cylindrically Deleted Coordination Space for Point Robots Navigating on a Plane}

We consider point robots navigating on a plane.
Thus, in this case $\Gamma_i$, the configuration space of the $i^{th}$ robot, is the Euclidean plane coordinatized by $(x_i,y_i)$. A collision set $\Delta_{i,j} = \{(x_i,y_i,x_j,y_j) ~|~ x_i\!=\!x_j, y_i\!=\!y_j\} \subset \Gamma_i \times \Gamma_j \simeq \mathbb{R}^4, ~1\!\leq\!i\!<\!j\!\leq\!N$, is a $2$-dimensional hyperplane embedded in the joint configuration space of the robots $i$ and $j$.

The joint configuration space  of $N$ robots (with the collision sets included) is $\overline{\Gamma} = \prod_{k=1}^N \Gamma_k \simeq \mathbb{R}^{2N}$.
The ``cylindrical'' obstacles in this configuration space created due to $\Delta_{i,j}$ are thus
\[ \mathcal{O}_{i,j} = \Delta_{i,j} \times \left( \prod_{k\neq i,j} \Gamma_k \right) = \{(x_1,y_1,x_2,y_2,\cdots,x_N,y_N) ~|~ x_i\!=\!x_j, y_i\!=\!y_j\} \subset \overline{\Gamma},\]
which are co-dimension $2$ subspaces (hyperplanes) embedded in the $2N$ dimensional joint configuration space.
Thus the coordination space is $X = \overline{\Gamma} - \bigcup_{1\leq i < j \leq N} \mathcal{O}_{i,j}$.


\mysubsubsection{Design of co-dimension $1$ manifolds, $U_*$:}
As before, we are interested in constructing a set of $(2N-1)$-dimensional 
sub-manifolds, $\{U_\alpha\}$, in $X$ such that removing all but one of these sub-manifolds will give us a space with fundamental group isomorphic to $\mathbb{Z}$. This would let us apply the generalized Van Kampen theorem as before by allowing us to construct words based on transverse intersection of paths with the surfaces, $U_*$ (which would be of co-dimension $1$ in $X$). 
We outline a general construction, and then show how that can be specialized for $N=3$.

Consider a single $(2N-2)$-dimensional obstacle $\mathcal{O}_{i,j}$, which is a hyperplane of co-dimension $2$ in $\overline{\Gamma}$.
Thus the homotopy group of $Y_{i,j} = \overline{\Gamma} - \mathcal{O}_{i,j}$ is isomorphic to $\mathbb{Z}$ and is generated by a $1$-dimensional loop that \emph{links} with $\mathcal{O}_{i,j}$ in this space.
Our first construction corresponding to the obstacle $\mathcal{O}_{i,j}$ is thus the following half-space:
\[
 \mathcal{U}_{i,j} = \{(x_1,y_1,x_2,y_2,\cdots,x_N,y_N) ~|~ x_i\!=\!x_j, y_i\!<\!y_j\}
\]
This a $(2N-1)$-dimensional (co-dimension $1$ in $\overline{\Gamma}$) half-hyperplane with $\mathcal{O}_{i,j}$ at its boundary. 
However, for another pair of indices, $(i',j'), 1\!\leq\!i'\!<\!j'\!\leq\!N$, the obstacle $\mathcal{O}_{i',j'}$ in general intersects $\mathcal{U}_{i,j}$ in a $(2N-3)$-dimensional half-hyperplane, and the space $\mathcal{U}_{i,j}$ intersects $\mathcal{U}_{i,j}$ in a $(2N-2)$-dimensional half-hyperplane.
In particular,
\[\mathcal{U}_{i,j} \cap \mathcal{O}_{i',j'} = \{(x_1,y_1,x_2,y_2,\cdots,x_N,y_N) ~|~ x_i\!=\!x_j, x_{i'}\!=\!x_{j'}, y_i\!<\!y_j, y_{i'}\!=\!y_{j'}\}\]
is of co-dimension $2$ in $\mathcal{U}_{i,j}$, and,
\[ 
\mathcal{U}_{i,j} \cap \mathcal{U}_{i',j'} = \{(x_1,y_1,x_2,y_2,\cdots,x_N,y_N) ~|~ x_i\!=\!x_j, x_{i'}\!=\!x_{j'}, y_i\!<\!y_j, y_{i'}\!<\!y_{j'}\}\]
is of co-dimension $1$ in $\mathcal{U}_{i,j}$ (a half-hyperplane) and is non-empty. This is schematically shown in Figure~\ref{fig:O-U-intersections}.


\begin{lemma}
 $X - \bigcup_{1\leq i<j\leq N} \mathcal{U}_{i,j} ~=~ \overline{\Gamma} - \bigcup_{1\leq i<j\leq N} \mathcal{O}_{i,j} - \bigcup_{1\leq i<j\leq N} \mathcal{U}_{i,j}$ ~is path connected.
\end{lemma}

\begin{proofsketch}
 Suppose $\mathbf{p}^0 = (x^0_1,y^0_1,x^0_2,y^0_2,\cdots,x^0_N,y^0_N), \mathbf{p}^1 = (x^1_1,y^1_1,x^1_2,y^1_2,\cdots,x^1_N,y^1_N) \in X$. If $x^0_i < x^0_j$, but $x^1_i \geq x^1_j$, then whichever path, $t\mapsto \mathbf{p}^t, ~t\in[0,1]$, is chosen to connect $\mathbf{p}^0$ and $\mathbf{p}^1$, there will be a $\tau$ for which $x^\tau_i = x^\tau_j$. However, if at this point, if $y^\tau_i \leq y^\tau_j$, then the path will be intersecting $\mathcal{O}_{i,j}$ or $\mathcal{U}_{i,j}$. Clearly, this can be prevented by simply altering the $y^t_i$ and $y^t_j$ in a small neighborhood of $t=\tau$, without altering any other coordinate and hence the alteration itself not leading to intersection with any other $\mathcal{U}_{i',j'}$. Thus an arbitrary pair of points in $X$ an be connected using a path that does not intersect any of the $\mathcal{U}_{i,j}$'s.
\end{proofsketch}


\begin{lemma}
%
 Any loop in $\overline{\Gamma}$ linked only with $\mathcal{O}_{i,j}$ can be homotoped into any other loop linked only with $\mathcal{O}_{i,j}$, through a sequence of loops linked to $\mathcal{O}_{i,j}$,
 \begin{enumerate}[leftmargin=20pt,labelindent=0pt,itemindent=0pt,labelsep=2pt]
  \item[i.] ~without intersecting $\mathcal{O}_{i',j'}$.
  \item[ii.] ~not without intersecting $\mathcal{U}_{i',j'}$ if either $i\!=\!i', ~j'\!<\!j$ or $i\!<\!i',~j\!=\!j'$, otherwise without intersecting $\mathcal{U}_{i',j'}$.
 \end{enumerate}

%
\end{lemma}
\begin{proofsketch}
The proof is based on being able to construct, or an obstruction to constructing, a homotopy satisfying certain properties between \emph{small} loops (in a tubular neighborhood of $\mathcal{O}_{i,j}$) lying in a plane transverse to $\mathcal{O}_{i,j}$ and linking to it.

\vspace{0.1in}
\noindent\textit{Case I: Distinct $i,j,i'$ and $j'$:}

 We first consider the case when $i,j,i'$ and $j'$ are all distinct.
 Consider a $2$-dimensional affine plane transverse to $\mathcal{O}_{i,j}$ described by $\mathcal{P}_{i,j}(c_*,d_*) = \{(x_1,y_1,x_2,y_2,\cdots,x_N,y_N) ~|~ x_i+x_j=c_{ij}, y_i+y_j=d_{ij},~ x_k=c_k,y_k=d_k, \forall k\neq i,j \}$ (where $c_*$ and $d_*$ refers to the set of parameters describing the plane), and coordinatized by $X_{ij} = x_i - x_j$ and $Y_{ij} = y_i - y_j$ (so that $\mathcal{O}_{i,j}$ intersects the plane at its origin).
 
 The intersection of this plane with $\mathcal{O}_{i',j'}$ is, in general, empty except for carefully chosen values for the parameters (in particular, for parameters such that $c_{i'}\!=\!c_{j'},d_{i'}\!=\!d_{j'}$), when they intersect over the entire $\mathcal{P}_{i,j}(c_*,d_*)$.
 Likewise, the intersection of this plane with $\mathcal{U}_{i',j'}$ is, in general, empty except when the parameters are such that $c_{i'}\!=\!c_{j'},d_{i'}\!<\!d_{j'}$), when, again, they intersect at the entire $\mathcal{P}_{i,j}(c_*,d_*)$.
 
 Given two such affine planes, $\mathcal{P}_{i,j}(c^0_*,d^0_*)$ and $\mathcal{P}_{i,j}(c^1_*,d^1_*)$, for two different sets of parameters, one can easily choose a path, $t \mapsto (c^t_*,d^t_*), ~t\in[0,1]$, through the parameter space avoiding $c^t_{i'}=c^t_{j'},d^t_{i'}=d^t_{j'}$ simultaneously at any $t$. This gives a homotopy for a loop in $\mathcal{P}_{i,j}(c^0_*,d^0_*)$ around the origin (linking with$\mathcal{O}_{i,j}$) to a loop in $\mathcal{P}_{i,j}(c^1_*,d^1_*)$ around the origin, without intersecting $\mathcal{O}_{i',j'}$.
 
 Similarly, it is possible to choose the path such that $c^t_{i'}=c^t_{j'}$ and $d^t_{i'}<d^t_{j'}$ does not happen simultaneously.
%

\vspace{0.1in}
\noindent\textit{Case II: $i,j,i'$ and $j'$ not all distinct:}

Next consider the case when $i,j,i',j'$ are not distinct. Choose $1\leq m<n<p\leq N$ to be the non-distinct indices such that either ~$i\!=\!i'\!=\!m < j'\!=\!n < j\!=\!p$, ~or ~$i\!=\!m<i'\!=\!n<j\!=\!j'\!=\!p$. Then the possible obstacles are $\mathcal{O}_{m,n}$, $\mathcal{O}_{n,p}$ and $\mathcal{O}_{m,p}$, with the indices of any pair of obstacle not all distinct.

Consider the plane $\mathcal{P}_{m,n}(c_*,d_*)$, on which $x_m+x_n=c_{mn},y_m+y_n=d_{mn}$ and $x_p=c_p,y_p=d_p$. This intersects obstacle $\mathcal{O}_{n,p}$ at points where $x_n=x_p=c_p$ and $y_n=y_p=d_p$. This gives $x_m=c_{mn}-c_p, y_m=d_{mn}-d_p$. Thus, on the plane, the coordinates of the intersection point are $X_{mn} = c_{mn}-2c_p, Y_{mn} = d_{mn}-2d_p$.
Once again, it is possible to choose the path $(c^t_*,d^t_*)$ from $(c^0_*,d^0_*)$ to $(c^1_*,d^1_*)$ such that $c^t_{mn}-2c^t_p$ and $d^t_{mn}-2d^t_p$ are not simultaneously zero for any $t\in[0,1]$.
This argument also holds for the other two pairs of obstacles.

Again,
using the chosen coordinates on $\mathcal{P}_{m,n}(c_*,d_*)$, 
$\mathcal{P}_{m,n}(c_*,d_*)$ intersects $\mathcal{U}_{n,p}$ at the ray $X_{mn} = c_{mn}-2c_p, Y_{mn} > d_{mn}-2d_p$,
and intersects $\mathcal{U}_{m,p}$ at the ray $X_{mn} = c_{mn}-2c_p, Y_{mn} < -(d_{mn}-2d_p)$.
Because the rays point in opposite directions along the $X_{mn}$ axis, 
given the parameters $(c^0_*,d^0_*)$ and $(c^1_*,d^1_*)$, it is always possible to find the path $(c^t_*,d^t_*)$ (in particular, $d^t_p$) such that the rays of intersection with $\mathcal{U}_{n,p}$ and $\mathcal{U}_{m,p}$ on $\mathcal{P}_{m,n}(c^t_*,d^t_*)$ does not pass through the origin (if $c^0_{mn}-2c^0_p$ and $c^1_{mn}-2c^1_p$ are of opposite signs, and if $c^\tau_{mn}-2c^\tau_p = 0$ for some $\tau\in[0,1]$, this can be achieved by choosing $d^\tau_{mn}-2d^\tau_p > 0$).
Similar argument holds for intersection of $\mathcal{P}_{n,p}(c_*,d_*)$ with $\mathcal{U}_{m,n}$ and $\mathcal{U}_{m,p}$.

However, $\mathcal{P}_{m,p}(c_*,d_*)$ intersects $\mathcal{U}_{m,n}$ at the ray $X_{mp} = 2 c_n - c_{mp}, Y_{mp} < 2 d_n - d_{mp}$, and $\mathcal{U}_{n,p}$ at the ray $X_{mp} = 2 c_n - c_{mp}, Y_{mp} < -(2 d_n - d_{mp})$. These rays point in the same direction. 
Clearly, given the parameters $(c^0_*,d^0_*)$ and $(c^1_*,d^1_*)$, if $2 c^0_n - c^0_{mp}$ and $2 c^1_n - c^1_{mp}$ are of opposite sign, it is not possible to find the path $(c^t_*,d^t_*)$ such that neither of these rays of intersection do not pass through the origin in $\mathcal{P}_{m,p}(c^t_*,d^t_*)$.
\end{proofsketch}

The consequence of the above Lemma is that the obstruction to fundamental group of $ X_{m,p} = X - \cup_{(i',j')\neq(m,p), 1\leq i' < j' \leq N} \mathcal{U}_{i',j'}$ being $\mathbb{Z}$ are the intersections of $\mathcal{U}_{m,p}$ with $\mathcal{U}_{m,n}$ and $\mathcal{U}_{n,p}$, for every $n$ such that $m<n<p$.
This leads us to split $\mathcal{U}_{m,p}$ by the the hyperplane at which it is intersected by $\mathcal{U}_{m,n}$ or $\mathcal{U}_{n,p}$, for every $m<n<p$.
It can however be noted that for $1\!\leq \!m\!<\!n\!<\!p\!\leq\! N$, ~$\mathcal{U}_{m,p} \cap \mathcal{U}_{m,n}$ and $\mathcal{U}_{m,p} \cap \mathcal{U}_{n,p}$ are subsets of the same $(2N-2)$-dimensional hyperplane. 
Thus we have the following splitting for $\mathcal{U}_{m,p}$ only due to its intersection with $\mathcal{U}_{m,n}$ and $\mathcal{U}_{n,p}$:
\begin{eqnarray}
 \mathcal{U}_{m,p/-} & = & \{(x_1,y_1,x_2,y_2,\cdots,x_N,y_N) ~|~ x_m\!=\!x_p \leq x_n, ~y_m\!<\!y_p\} \nonumber \\
 \mathcal{U}_{m,p/+} & = & \{(x_1,y_1,x_2,y_2,\cdots,x_N,y_N) ~|~ x_m\!=\!x_p \geq x_n, ~y_m\!<\!y_p\} \nonumber
\end{eqnarray}
In general, 
we have partitions of the form
{\small \begin{equation} 
 \displaystyle \mathcal{U}_{m,p/\sigma_{m+1},\sigma_{m+2},\cdots,\sigma_{p-1}} 
 = \left\{(x_1,y_1,x_2,y_2,\cdots,x_N,y_N) ~|~ ~y_m\!<\!y_p, 
				\left( \begin{array}{l} x_m\!=\!x_p \leq x_n ~\text{if $\sigma_n =$`$-$'} \\ x_m\!=\!x_p \geq x_n ~\text{if $\sigma_n =$`$+$'} \end{array}, \forall ~m<n<p \right) \right\}
 \label{eq:coordination-Us}
\end{equation}}
Note that the $\sigma$'s are indexed by integers from $m\!+\!1$ to $p\!-\!1$ just for convenience (instead of indexing using $1,2\cdots,p-m-1$).


\vspace{0.1in}
The following is a direct consequence

\begin{corollary}
The fundamental group of
\[ X_{i,j/\sigma_{i+1},\sigma_{j+2},\cdots,\sigma_{j-1}} \quad := \quad X ~~- \bigcup_{\begin{array}{c} \scriptstyle{ (i',j',\varsigma_{i'+1},\cdots,\varsigma_{j'-1}) \neq (i,j,\sigma_{i+1},\cdots,\sigma_{j-1}),} \\ \scriptstyle{ 1\leq i' < j' \leq N, ~~\varsigma_k\in\{+,-\}\forall i'<k'<j'} \end{array}} \!\!\!\! \mathcal{U}_{i',j'/\varsigma_{i'+1},\cdots,\varsigma_{j'-1}}\]
is isomorphic to $\mathbb{Z}$.
\end{corollary}

\begin{figure}
   \begin{center}
   \subfigure[Schematic illustration of $\mathcal{O}_{i,j}$, $\mathcal{U}_{i,j}$ and its intersection with $\mathcal{O}_{i',j'}$, $\mathcal{U}_{i',j'}$.]{ \label{fig:O-U-intersections}
       \includegraphics[height=0.35\textwidth, trim=0 0 0 0, clip=true]{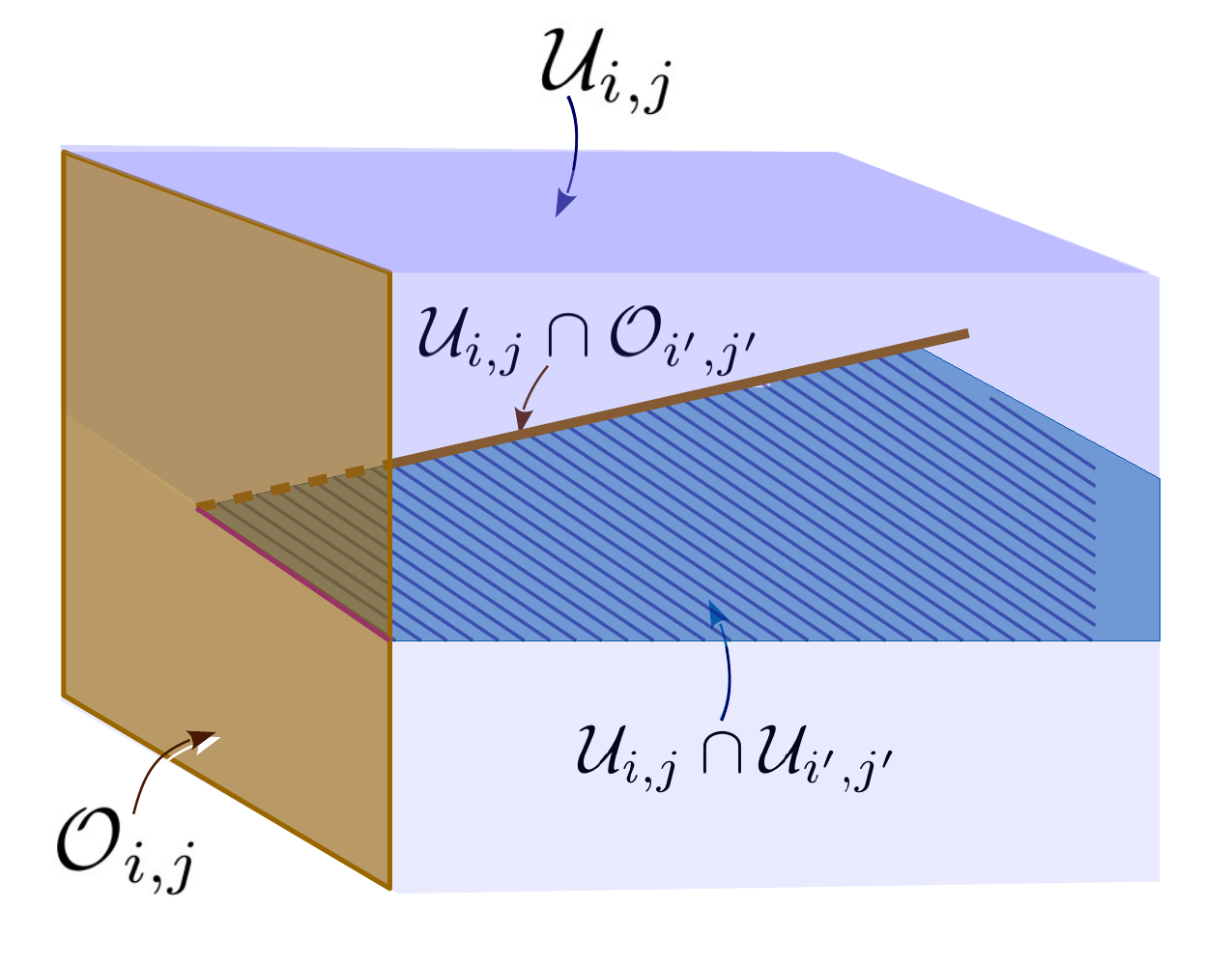}}
   \hspace{0.1in}
  \subfigure[Plane $\mathcal{P}_{m,n,p}(c_*,d_*)$ transverse to $\mathcal{U}_{m,p} \cap \mathcal{U}_{m,n} \cap \mathcal{U}_{n,p}$.]{ \label{fig:mnp-word}
       \includegraphics[height=0.35\textwidth, trim=0 0 0 0, clip=true]{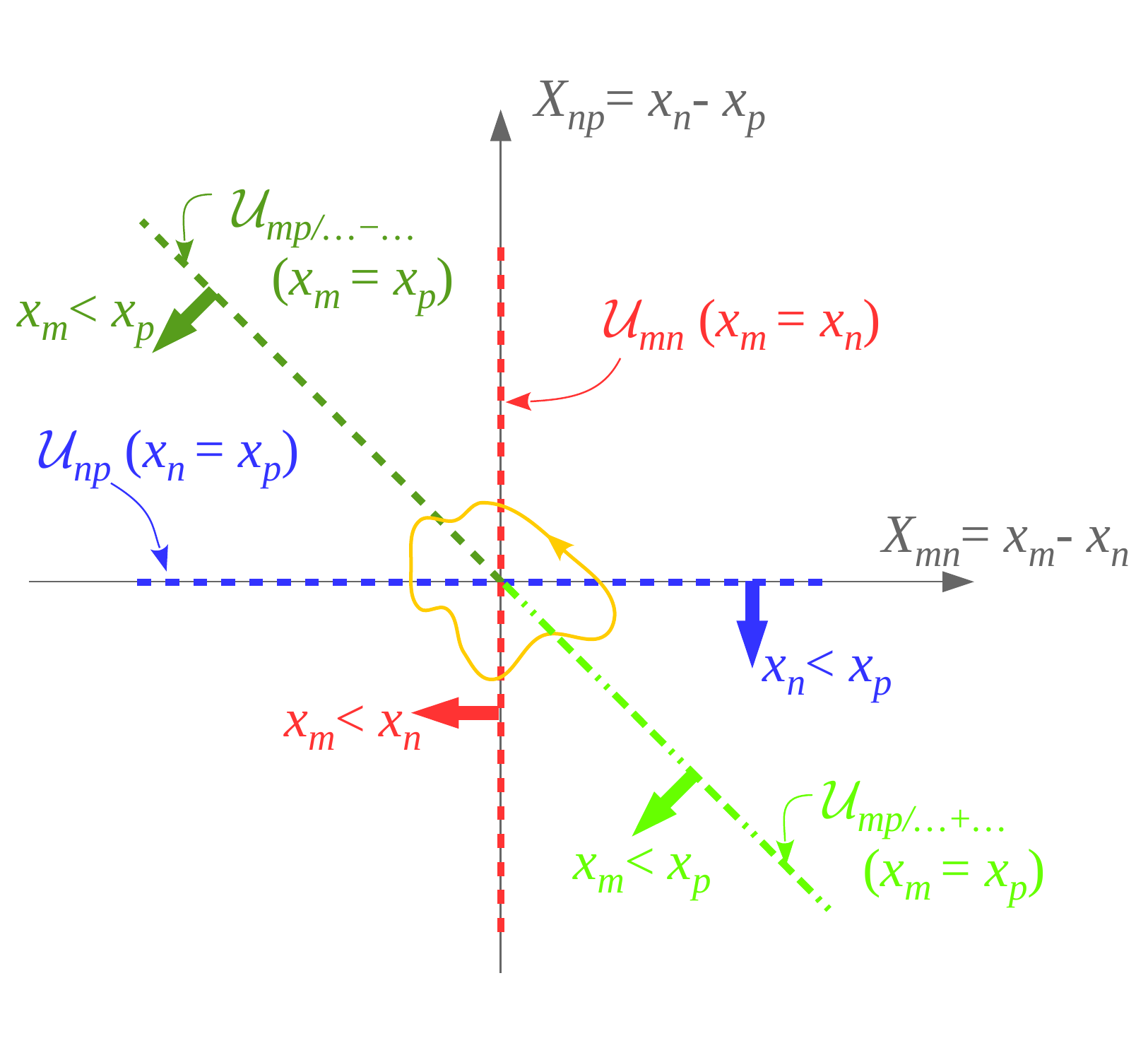}}
  \end{center}
\caption{\revA{Schematic illustration of the intersection of the constructed half-spaces, $\mathcal{U}$.}} \label{fig:OU-mnp}
\end{figure}

Thus, the $(2N-1)$ dimensional spaces, $\mathcal{U}_{m,p/\sigma_{m+1},\sigma_{m+2},\cdots,\sigma_{p-1}}$, where $1\leq i < j \leq N$, and $\sigma_k$ is either `$+$' or `$-$', satisfy the conditions of Proposition~\ref{prop:van-kampen-general}.
The set $\mymathsf{U}$ is freely generated by letters, $u_{m,p/\sigma_{m+1},\sigma_{m+2},\cdots,\sigma_{p-1}}$, corresponding to each of these spaces.

\mysubsubsection{The Relation Set:}
The relation set, $\mymathsf{R}$, contains the following types of words:
\nopagebreak[4]
\begin{enumerate}[leftmargin=20pt,labelindent=0pt,itemindent=0pt,labelsep=2pt]
 \item[i.] Due to the intersection of $\mathcal{U}_{m,p}$, $\mathcal{U}_{m,n}$ and $\mathcal{U}_{n,p}$ ($m<n<p$):~ 
           $\mathcal{U}_{m,p}$, $\mathcal{U}_{m,n}$ and $\mathcal{U}_{n,p}$ intersect at a common $(2N-2)$-dimensional hyperplane on which $x_m=x_n=x_p$, and for $y_m<y_n<y_p$. Let $\mathcal{P}_{m,n,p}(c_*,d_*)$ be a plane transverse to this intersection, coordinatized by $X_{mn} = x_m - x_n$ and $X_{np} = x_n - x_p$, so that $\mathcal{U}_{m,p} \cap \mathcal{U}_{m,n} \cap \mathcal{U}_{n,p}$ intersects it at the origin on this plane. A loop around the origin in $\mathcal{P}_{m,n,p}(c_*,d_*)$ (Figure~\ref{fig:mnp-word}) has the following word:
           \begin{equation}\label{eq:coordination-rel-word-1a} \begin{array}{l}
              \text{``} u_{m,n/\alpha_{m+1},\cdots,\alpha_{n-1}} ~\cdot~ 
              u_{m,p/\sigma_{m+1},\cdots,\sigma^{(1)}_n=-,\cdots,\sigma_{p-1}} ~\cdot~ 
              u_{n,p/\beta_{n+1},\cdots,\beta_{p-1}} ~\cdot~ \\~\quad\quad\quad
              u_{m,n/\alpha_{m+1},\cdots,\alpha_{n-1}}^{-1} ~\cdot~ 
              u_{m,p/\sigma_{m+1},\cdots,\sigma^{(2)}_n=+, \cdots,\sigma_{p-1}}^{-1} ~\cdot~ 
              u_{n,p/\beta_{n+1},\cdots,\beta_{p-1}} \text{''}
           \end{array}\end{equation}
            for any choice of signs $\alpha_a,\beta_b \in \{+,-\}, m\!<\!a\!<\!n, n\!<\!b\!<\!p$, ~and 
            ~$\sigma_k = \left\{ \begin{array}{l} \alpha_k, \text{ if $m<k<n$} \\ \beta_k, \text{ if $n<k<p$}  \end{array} \right.$ (in order to ensure non-empty intersection).
            Note that the sign value of $\sigma_n$ in the second letter of the word above is `$-$', while it is `$+$' in the fifth letter.

	    It can be noted that there exists non-empty intersection $\mathcal{U}_{m,p} \cap \mathcal{U}_{m,n}$ and $\mathcal{U}_{m,p} \cap \mathcal{U}_{n,p}$ that are outside the aforementioned common intersection of the three (the former when $y_m<y_n,y_m<y_p\leq y_n$ and the later when $y_n<y_p,y_n\leq y_m<y_p$). These intersections 
	    result in the following additional words in the relation set:
	    \begin{equation}\label{eq:coordination-rel-word-1b} \begin{array}{l}
              \text{``} u_{m,n/\alpha_{m+1},\cdots,\alpha_{n-1}} ~\cdot~ 
              u_{m,p/\sigma_{m+1},\cdots,\sigma^{(1)}_n=-,\cdots,\sigma_{p-1}} ~\cdot~ \\~\quad\quad\quad 
              u_{m,n/\alpha_{m+1},\cdots,\alpha_{n-1}}^{-1} ~\cdot~ 
              u_{m,p/\sigma_{m+1},\cdots,\sigma^{(2)}_n=+, \cdots,\sigma_{p-1}}^{-1} ~\cdot~ 
           \end{array}\end{equation}
           and
           \begin{equation}\label{eq:coordination-rel-word-1c} \begin{array}{l}
              \text{``} 
              u_{m,p/\sigma_{m+1},\cdots,\sigma^{(1)}_n=-,\cdots,\sigma_{p-1}} ~\cdot~ 
              u_{n,p/\beta_{n+1},\cdots,\beta_{p-1}} ~\cdot~ \\~\quad\quad\quad
              u_{m,p/\sigma_{m+1},\cdots,\sigma^{(2)}_n=+, \cdots,\sigma_{p-1}}^{-1} ~\cdot~ 
              u_{n,p/\beta_{n+1},\cdots,\beta_{p-1}} \text{''}
           \end{array}\end{equation}

  \item[ii.] Due to the intersection of $\mathcal{U}_{i,j}$ and $\mathcal{U}_{i',j'}$ (with $i\leq i'$), 
              where $i,j,i',j'$ are all distinct:~
	  Without loss of generality, assume either $i<i'$ or $i=i', j<j'$.
	    $\mathcal{U}_{i,j}$ and $\mathcal{U}_{i',j'}$ intersect at a $(2N-2)$-dimensional hyperplane on which $x_i\!=\!x_j, x_{i'}\!=\!x_{j'}$.
	    The letters corresponding to these spaces simply commute. Thus we have the words:
	    \begin{equation}\label{eq:coordination-rel-word-2} \begin{array}{l}
	      \text{``} u_{i,j/\sigma_{i+1},\cdots,\sigma_{j-1}} ~\cdot~ u_{i',j'/\gamma_{i'+1},\cdots,\gamma_{j'-1}} ~\cdot~ \\~\quad\quad\quad
	      u_{i,j/\sigma_{i+1},\cdots,\sigma_{j-1}}^{-1} ~\cdot~ u_{i',j'/\gamma_{i'+1},\cdots,\gamma_{j'-1}}^{-1} \text{''}
	    \end{array}\end{equation}
	    for all $\sigma_k,\gamma_l\in \{+,-\}$

\end{enumerate}

\noindent
The relation set, $\mymathsf{R}$, thus consists of all the words of the forms described in \eqref{eq:coordination-rel-word-1a}, \eqref{eq:coordination-rel-word-1b}, \eqref{eq:coordination-rel-word-1c} and \eqref{eq:coordination-rel-word-2}.

\mysubsubsection{Explicit Example for $N=3$} \label{sec:coordination-3robot-example}
We consider the simple, yet non-rival case of $N=3$ (coordination space of $3$ robots navigating on a plane).
The letters in $\mymathsf{U}$ are ~$u_{1,2}$, ~$u_{2,3}$, ~$u_{1,3/+}$ and ~$u_{1,3/-}$
corresponding to respectively crossing of the manifolds ~$\mathcal{U}_{1,2} = \{\mathbf{p} ~|~ x_1=x_2,y_1<y_2\}$, ~$\mathcal{U}_{2,3} = \{\mathbf{p} ~|~ x_2=x_3,y_2<y_3\}$, ~$\mathcal{U}_{1,3/+} = \{\mathbf{p} ~|~ x_1=x_3 > x_2,y_1<y_3\}$ and ~$\mathcal{U}_{1,3/-} = \{\mathbf{p} ~|~ x_1=x_3 < x_2,y_1<y_3\}$.
The relation set consists of the following words:
\[\begin{array}{rl}
 \mymathsf{R} = & \big\{~~ 
                  u_{1,2} ~~ u_{1,3/-} ~~ u_{2,3} ~~ u_{1,2}^{-1} ~~ u_{1,3/+}^{-1} ~~ u_{2,3}^{-1}, \\
		    & \qquad u_{1,2} ~~ u_{1,3/-} ~~  u_{1,2}^{-1} ~~ u_{1,3/+}^{-1}, \\
		    & \qquad\quad u_{1,3/-} ~~ u_{2,3} ~~ u_{1,3/+}^{-1} ~~ u_{2,3}^{-1} 
               ~~\big\}
  \end{array}
\]

For simplicity, define $a = u_{1,2}$, ~$b = u_{1,3/+}$, ~$c = u_{1,3/-}$, ~$d = u_{2,3}$.
Rewriting the relation set, $\mymathsf{R} = \{acda^{-1}b d^{-1}, a c a^{-1} b, c d b d^{-1}\}$.
This can be easily shown (by isolating $b$ by setting the last relation to identity, and substituting it in the other two relations) to be isomorphic to the group
$< a,c,d ~|~ acd(dac)^{-1}, cda (dac)^{-1}>$.
Which in turn (using substitution $p=cd,q=dac$) can be shown to be isomorphic to the group \newline 
$<a,p,q ~|~ ap(pa)^{-1}, aq(qa)^{-1}>$ ($a$ commutes with both $p$ and $q$, but $p$ and $q$ themselves do not commute). This is clearly the group $\mathbb{Z} \times (\mathbb{Z}*\mathbb{Z})$. In fact, simply using geometric arguments, it is easy to verify that the homotopy type of the coordination space of $3$ robots on a plane is indeed that of $\mathbb{S}^1 \times (\mathbb{S}^1 \vee \mathbb{S}^1)$ (see \citep{Arslan:Tro:2016} for example).
}

\mysection{Simulation Results}

\mysubsection{Knot and Link Complements}

Given obstacles $\mathcal{O}\subset \mathbb{R}^3$, and their ``skeletons'' ($1$-dimensional homotopy equivalents), $O\subseteq \mathcal{O}$ as polygons in $\mathbb{R}^3$, we first choose a projection map, $p: \mathbb{R}^3 \rightarrow \mathbb{R}^2$, for the knot/link diagram. With this information, we implemented the automated 
construction of the surfaces, $U_i$, for the Dehn presentation of the knot/link complement, and 
the symmetricized relation set, $\overline{\mymathsf{R}}$, 
by computing the self-intersections in $p(O)$. 
We then used a uniform cubical discretization of $\mathbb{R}^3-\mathcal{O}$ to construct the graph $G$ as a discrete representation of the free space, and in the $h$-augmented graph, $G_h$, we find trajectories from $(x_s,\ldq ~\rdq)$ to $(x_g,*)$. We then employ a curve shortening algorithm to shorten the obtained trajectories.
All our implementations were done in C++ programming language and visualization were done using OpenGL. The program ran on a laptop running on a Intel i7-4500U processor @ 1.80GHz with 8 GB memory.

Figure~\ref{fig:result-knot1} shows results in the complement of a trefoil knot. The inset figure shows the surfaces, $U_i$, used for Dehn presentation. The graph $G$ was constructed out of uniform $100 \!\times\! 100 \!\times\! 100$ cubical discretization of the environment.
\revA{As seen in the figure, using the proposed search-based algorithm we computed $5$ shortest paths in different homotopy classes and used curve-shortening algorithms to shorten them to obtain the piece-wise linear paths.}
The entire computation (computation of the surfaces, the symmetricized relation set $\mymathsf{R}$, and computation of the $5$ shortest trajectories) required about $8.1$\,s. 
Likewise, Figure~\ref{fig:result-link1} shows results in the complement of a simple Hopf link, with the same discretization of the environment, and total computation time of about $8.2$\,s.

Figure~\ref{fig:result-complex} shows a much more complex obstacle involving a torus knot linked to a genus-$2$ obstacle, and the entire computation of $20$ trajectories took about $2.6$\,s.
\revA{While in a complex space as this, the completeness of Dehn algorithm is not guaranteed, numerical computation still gives us reasonable results as can be seen in this simulation result.

It is interesting to observe that computing shortest paths in $5$ different classes in the environments of Figures~\ref{fig:result-knot1} and \ref{fig:result-link1} took significantly longer than the computation of $20$ classes in the environment of Figure~\ref{fig:result-complex}.
This is because in Figure~\ref{fig:result-complex}, although much more complex, the obstacles themselves occupy a large part of the environment. This significantly reduces the number of vertices and edges in the discrete graph representing the free configuration space than the trefoil knot and Hopf link complements. Thus the graph search runs much faster in the configuration space of Figure~\ref{fig:result-complex}.
}

\begin{figure}
   \begin{center}
  \subfigure[$5$ trajectories in a trefoil knot complement.]{ \label{fig:result-knot1}
   \fbox{\includegraphics[height=0.26\textwidth, trim=0 0 0 0, clip=true]{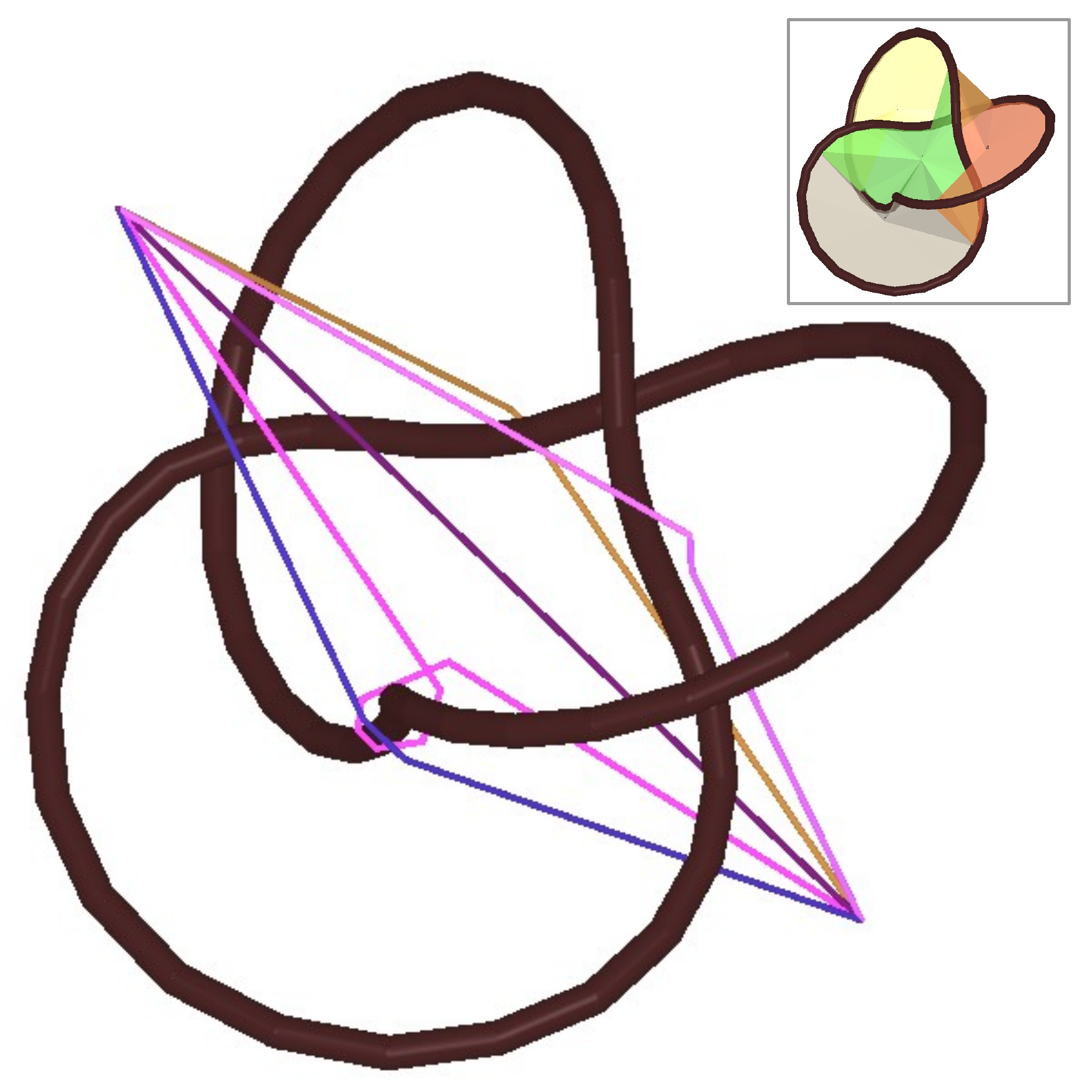}}
    } \hspace{0.01in}
  \subfigure[$5$ trajectories in a Hopf link complement.]{ \label{fig:result-link1}
   \fbox{\includegraphics[height=0.26\textwidth, trim=0 0 0 0, clip=true]{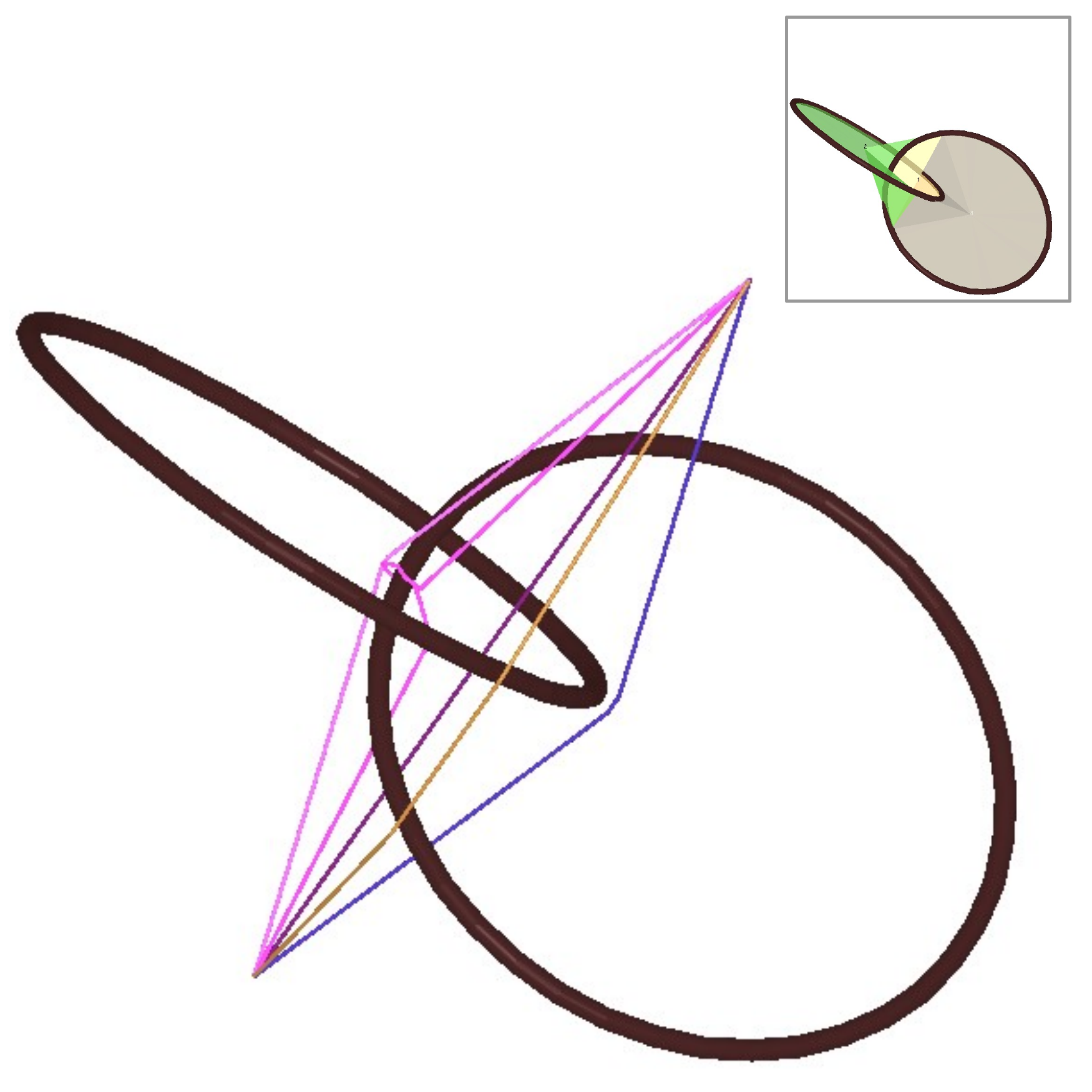}}
    } \hspace{0.01in}
  \subfigure[$20$ trajectories in the complement of a $(3,8)$ torus knot linked to a genus-$2$ torus.]{ \label{fig:result-complex}
   \fbox{\includegraphics[height=0.26\textwidth, trim=0 0 0 0, clip=true]{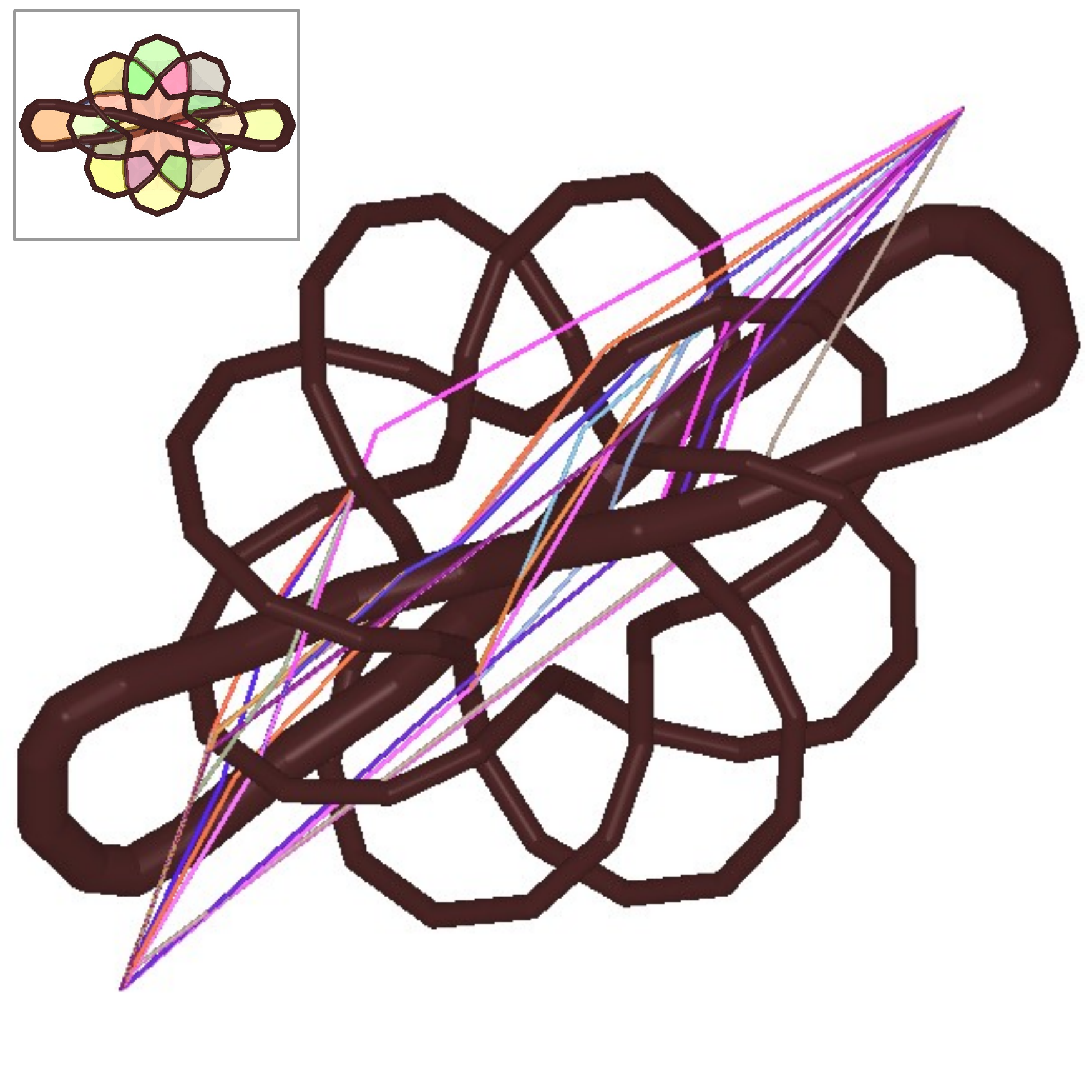}}
    }
  \end{center}
\caption{Optimal trajectories (in discrete graph representation, followed by curve shortening) in different homotopy classes in complements of knots and links. Insets show the surfaces, $U_i$.} \label{fig:results-knots-links}
\end{figure}

\changedJ{
\mysubsection{Coordination Space of $3$ Robots Navigating on a Plane}

Similar to the method described in Section~\ref{sec:graph-search}, we implemented the graph search based approach for finding optimal paths in the coordination space of $3$ robots navigating on a planar region.
We constructed a graph using a uniform hyper-cubical discretization of the configuration space, $\overline{\Gamma}=\mathbb{R}^6$, placing a vertex in each cell, and establishing edges between neighboring cells that correspond to robot motions parallel to the coordinate axes.
The optimality criteria in the A* search was chosen to be the sum of the total lengths of the paths (which, because of the chosen discretization, is the length due to the Manhattan metric on the plane) traversed by the robots. 
\revA{The $h$-augmented graph construction from this graph is similar to what has been described in Section~\ref{sec:graph-search}. However for checking whether or not two $h$-signatures are the same, we use the Dehn algorithm on the presentation described in Example~\ref{sec:coordination-3robot-example}.}

Figure~\ref{fig:results-coordination} \revA{(also, in attached multimedia file)} shows paths found in $5$ different homotopy classes for the coordination space of $3$ robots navigating on a plane. With the plane discretized into $7\times 7$ grid, and the degree of each vertex in the corresponding joint configuration space being $124$ (allowing each robot to move north, south, east or west, or stay in place), the computation of paths in $5$ different homotopy classes in the coordination space took about $31.7$\,s.
\revA{A careful observation of the figures will indeed reveal that the paths shown in each of the Figures~\ref{fig:results-coordination} (a)-(e) correspond to the robots taking distinct homotopy classes in the joint configuration space.}


\begin{figure}
   \begin{center}
  \subfigure[First homotopy class. \emph{Left-most figure:} Initial configuration of the robots, with the robots' indices labeled. \emph{Right-most figure:} Final configuration of the robots. Intermediate figures illustrate the paths taken by the robots. Word corresponding to this class is ``$u_{1,3/-}^{-1} ~\cdot~ u_{1,2}^{-1}$'' ]{ \label{fig:result-coordination1-0}
   \fbox{\includegraphics[height=0.18\textwidth, trim=0 0 0 0, clip=true]{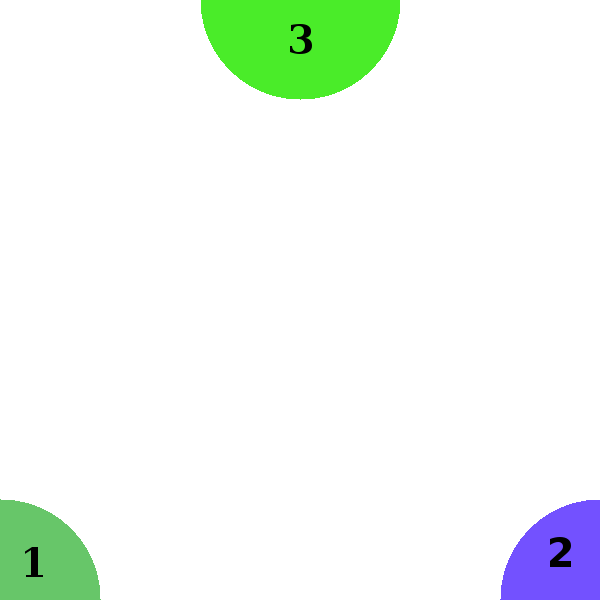}}
   \fbox{\includegraphics[height=0.18\textwidth, trim=0 0 0 0, clip=true]{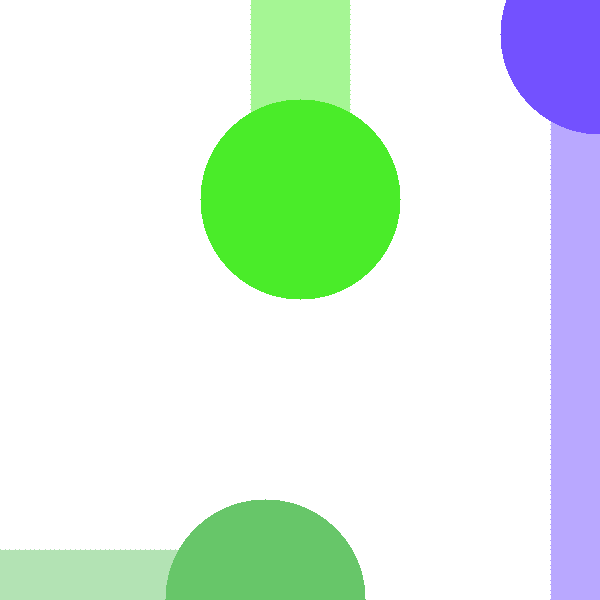}}
   \fbox{\includegraphics[height=0.18\textwidth, trim=0 0 0 0, clip=true]{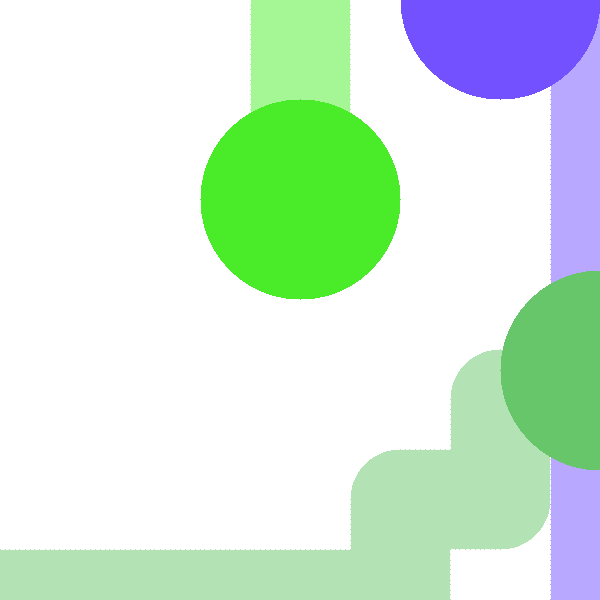}}
   \fbox{\includegraphics[height=0.18\textwidth, trim=0 0 0 0, clip=true]{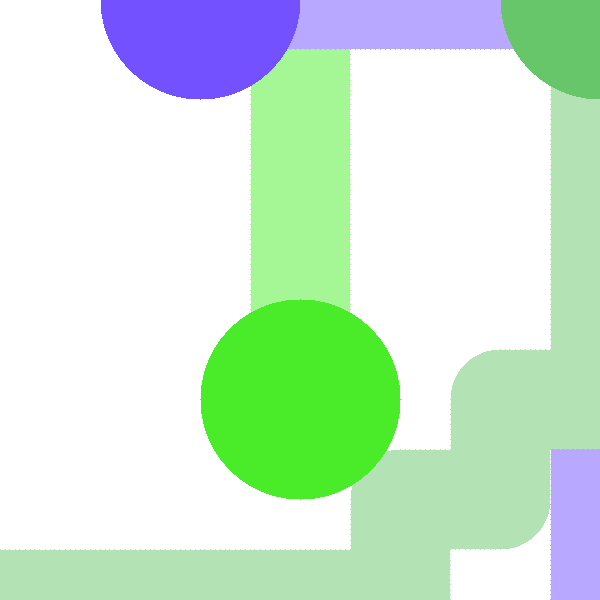}}
   \fbox{\includegraphics[height=0.18\textwidth, trim=0 0 0 0, clip=true]{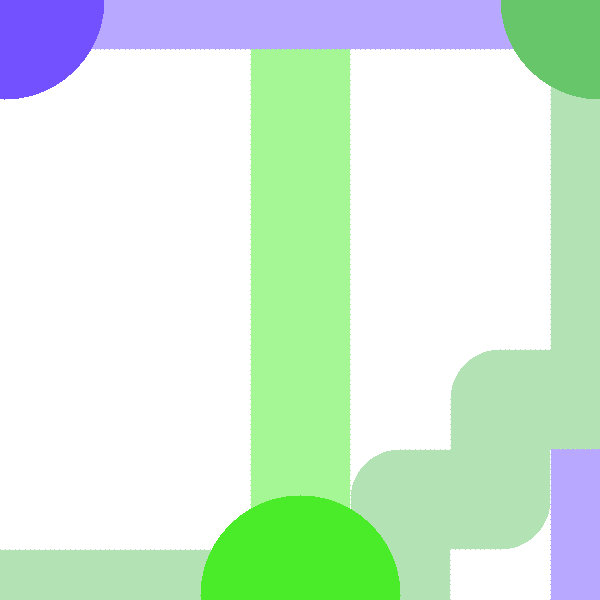}}
    }
   \subfigure[Second homotopy class corresponding to word ``$u_{2,3}$''.]{ \label{fig:result-coordination1-1}
   \fbox{\includegraphics[height=0.18\textwidth, trim=0 0 0 0, clip=true]{figures/Hx_init.png}}
   \fbox{\includegraphics[height=0.18\textwidth, trim=0 0 0 0, clip=true]{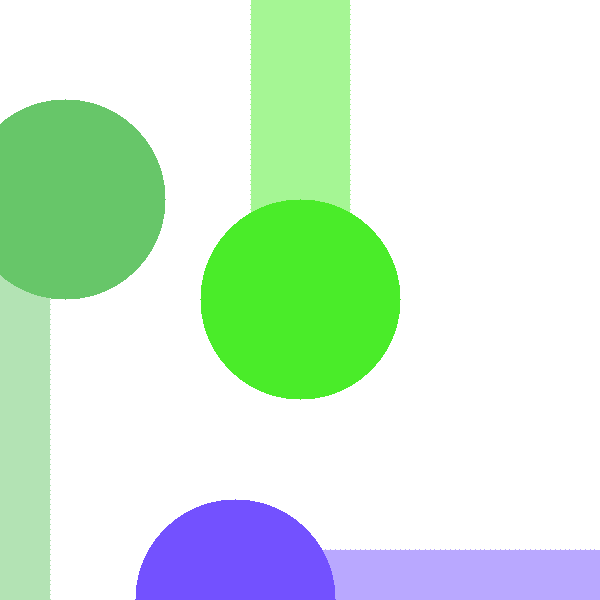}}
   \fbox{\includegraphics[height=0.18\textwidth, trim=0 0 0 0, clip=true]{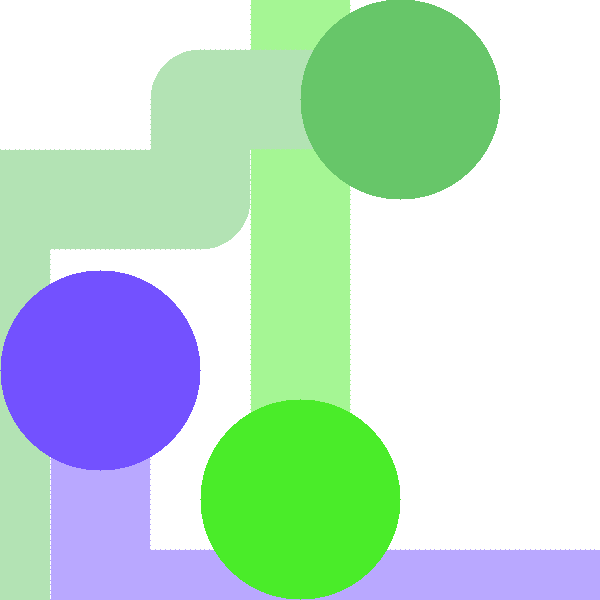}}
   \fbox{\includegraphics[height=0.18\textwidth, trim=0 0 0 0, clip=true]{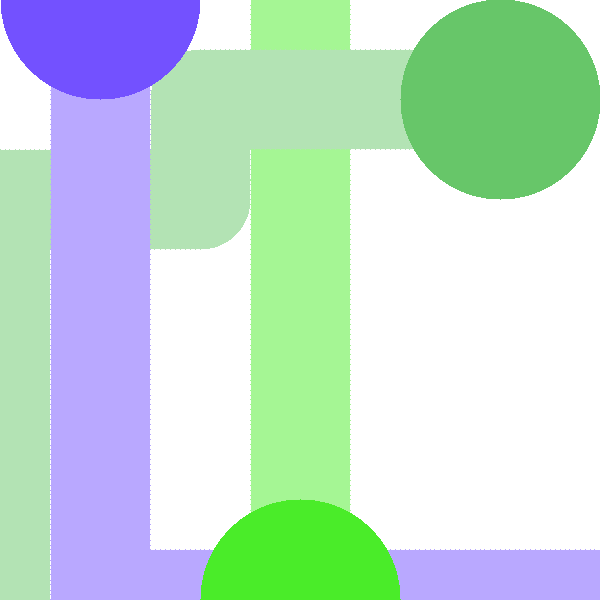}}
   \fbox{\includegraphics[height=0.18\textwidth, trim=0 0 0 0, clip=true]{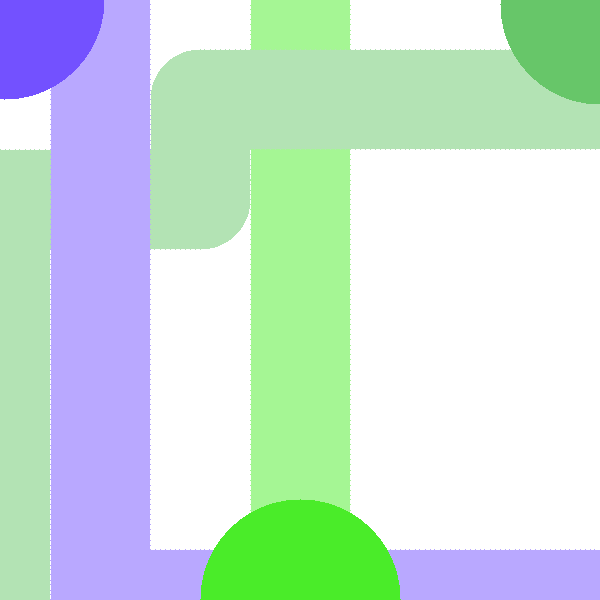}}
    }
   \subfigure[Third homotopy class corresponding to word ``$u_{2,3} ~\cdot~ u_{1,3/+}^{-1}$'']{ \label{fig:result-coordination1-2}
   \fbox{\includegraphics[height=0.18\textwidth, trim=0 0 0 0, clip=true]{figures/Hx_init.png}}
   \fbox{\includegraphics[height=0.18\textwidth, trim=0 0 0 0, clip=true]{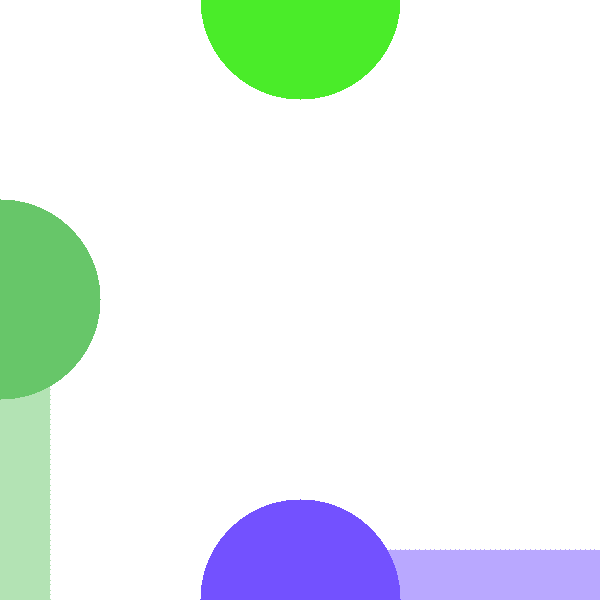}}
   \fbox{\includegraphics[height=0.18\textwidth, trim=0 0 0 0, clip=true]{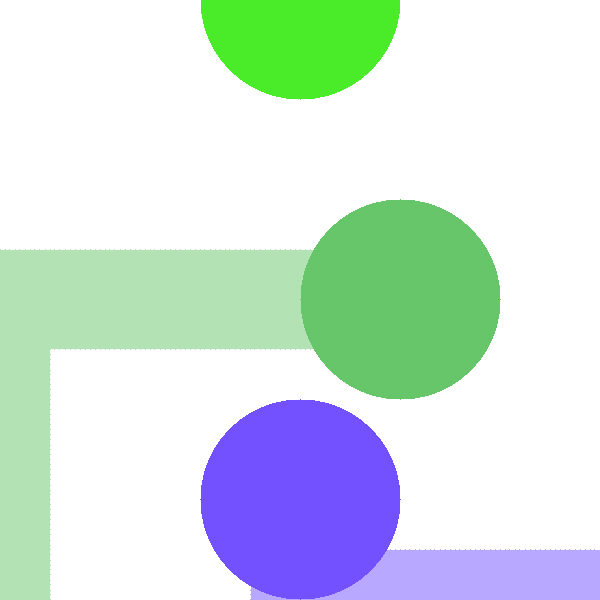}}
   \fbox{\includegraphics[height=0.18\textwidth, trim=0 0 0 0, clip=true]{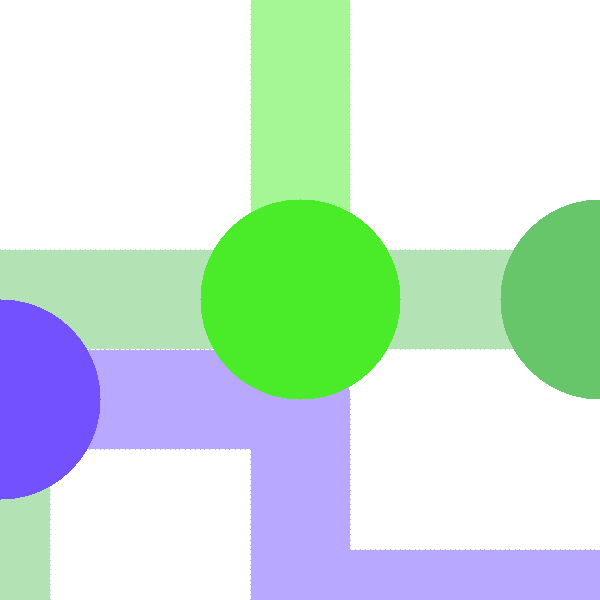}}
   \fbox{\includegraphics[height=0.18\textwidth, trim=0 0 0 0, clip=true]{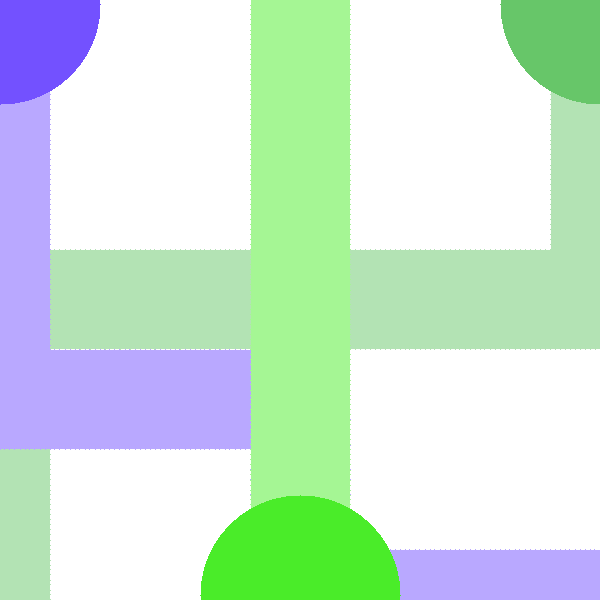}}
    }
    \subfigure[Fourth homotopy class corresponding to word ``$u_{2,3} ~\cdot~ u_{1,3/+}^{-1} ~\cdot~ u_{2,3}^{-1}$'']{ \label{fig:result-coordination1-2}
   \fbox{\includegraphics[height=0.18\textwidth, trim=0 0 0 0, clip=true]{figures/Hx_init.png}}
   \fbox{\includegraphics[height=0.18\textwidth, trim=0 0 0 0, clip=true]{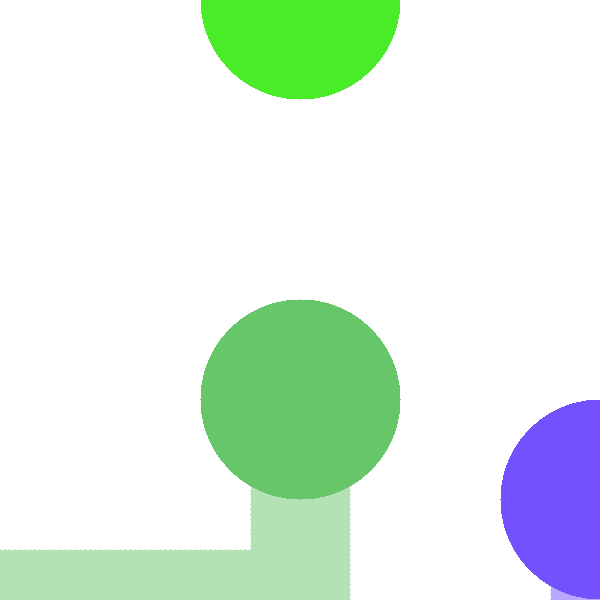}}
   \fbox{\includegraphics[height=0.18\textwidth, trim=0 0 0 0, clip=true]{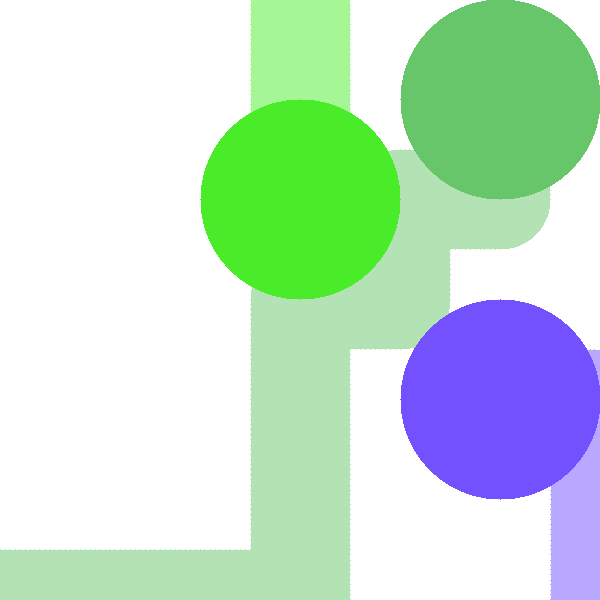}}
   \fbox{\includegraphics[height=0.18\textwidth, trim=0 0 0 0, clip=true]{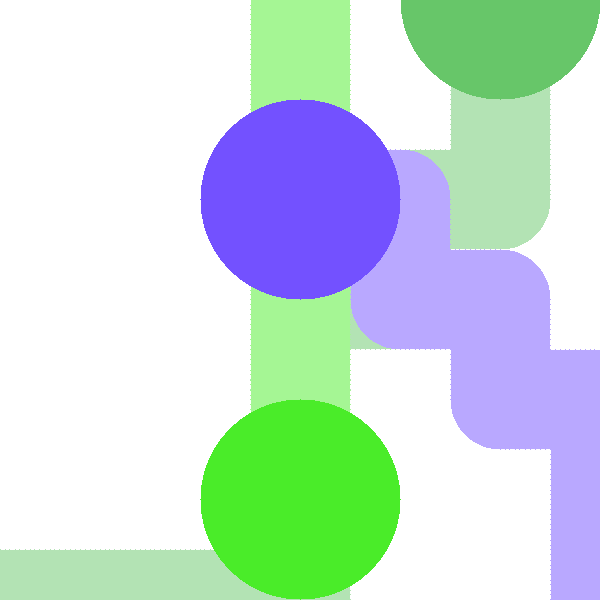}}
   \fbox{\includegraphics[height=0.18\textwidth, trim=0 0 0 0, clip=true]{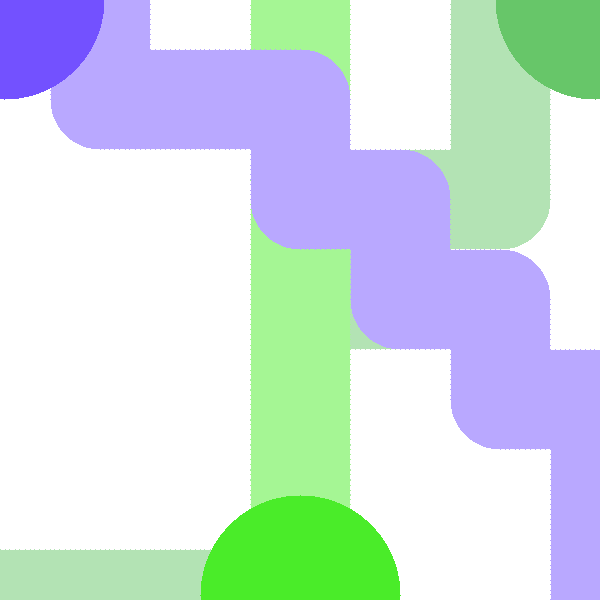}}
    }
  \subfigure[Fifth homotopy class corresponding to word ``$~$'']{ \label{fig:result-coordination1-2}
   \fbox{\includegraphics[height=0.18\textwidth, trim=0 0 0 0, clip=true]{figures/Hx_init.png}}
   \fbox{\includegraphics[height=0.18\textwidth, trim=0 0 0 0, clip=true]{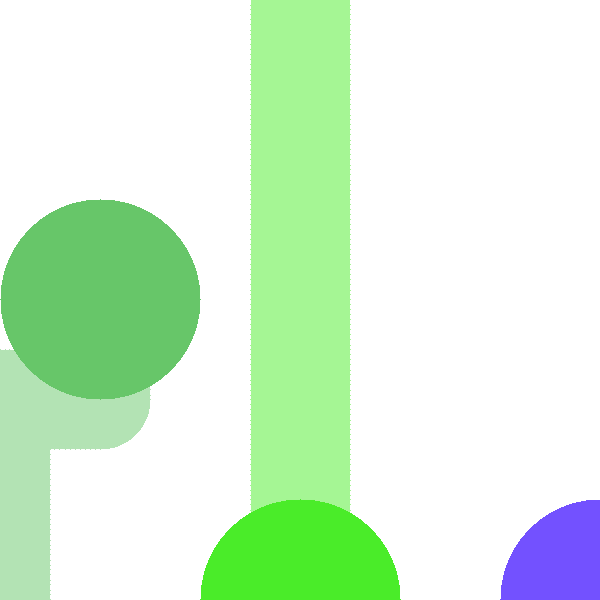}}
   \fbox{\includegraphics[height=0.18\textwidth, trim=0 0 0 0, clip=true]{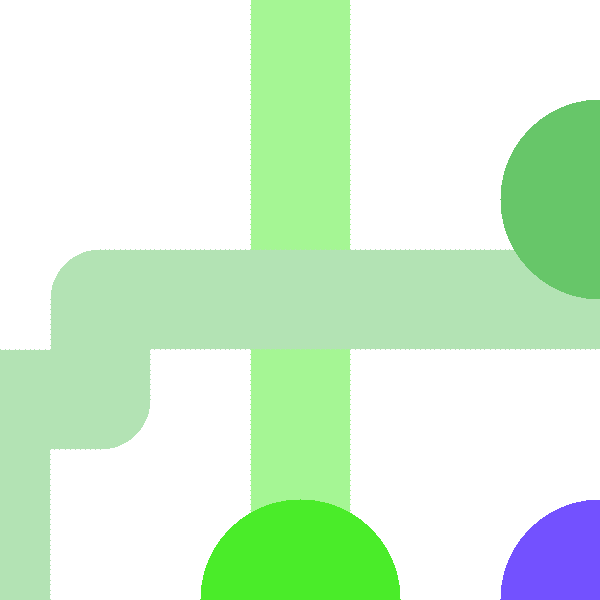}}
   \fbox{\includegraphics[height=0.18\textwidth, trim=0 0 0 0, clip=true]{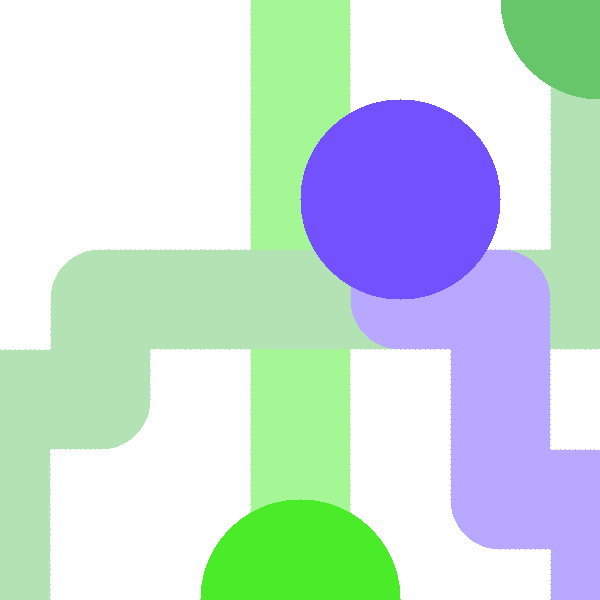}}
   \fbox{\includegraphics[height=0.18\textwidth, trim=0 0 0 0, clip=true]{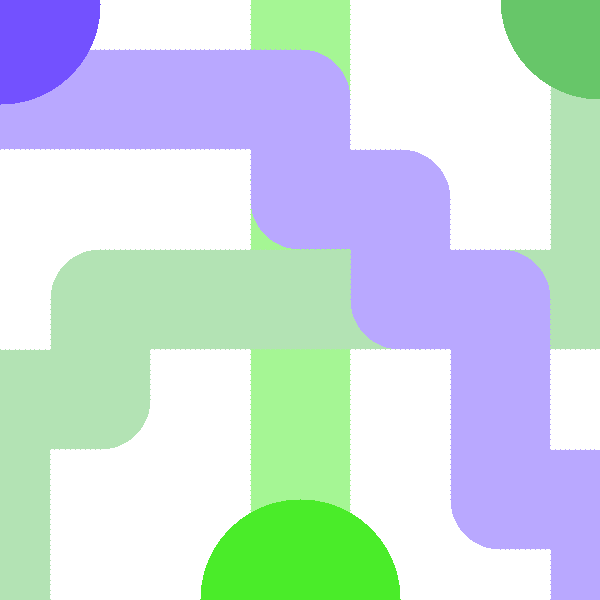}}
    }
  \end{center}
 \caption{Paths in $5$ different homotopy classes in the coordination space of $3$ robots navigating on a plane for an initial configuration to a final goal configuration. The classes found are in ascending order of the sum of the lengths of the paths of the three robots (where length is induced by the Manhattan metric on the plane).} \label{fig:results-coordination}
\end{figure}
}

\revA{
\mysubsection{Some Remarks on Number of Homotopy Classes:}

As described earlier, we compute the paths in the different homotopy classes by executing graph search algorithms (Dijkstra's/A*) on the $h$-augmented graph, $G_h$. As a result we obtain the paths in the different homotopy classes in order of their path lengths/costs.
Technically, the number of homotopy classes in any of the configuration spaces described earlier infinite. This is because a path can loop/wind round an obstacle arbitrarily many times, thus creating arbitrarily many different homotopy classes.
In the simulations we restrict ourselves to the computation of the first few homotopy classes in each configuration space since those are the classes that are most relevant in robotics applications.

}

\changedJ{

\mysection{Conclusion and Future Work}

We presented explicit construction of presentations of the fundamental group of two different classes of spaces: Knot/link complements in Euclidean $3$-dimensional spaces, and cylindrically deleted coordination space of multiple robots. We thus used graph search-based optimal path planning method for computing optimal paths in different homotopy classes in such environments.
\revA{These spaces are highly relevant in many robot motion planning problems, for example unmanned aerial vehicles (UAVs) navigating inside buildings, and multiple ground robots navigating on a plane. Being able to compute optimal paths in different homotopy classes allow us to efficiently solve motion planning problems for complex systems such as tethered robots and systems involving cables, and design strategies for effectively deploying groups of robots in coverage and exploration tasks.}

\revA{Moving forward, we will use the developed technique for homotopy path planning in 3-dimensional spaces with obstacles for computing optimal traversable paths for tethered unmanned aerial vehicles (UAVs).
One of the greatest challenges currently faced by UAVs is the limited battery life, which in turn limits the UAVs' flight time. Having the UAV tethered to a power supply will help alleviate the problem. But in that case one needs to carefully plan paths that satisfy the cable length constraints and prevents the cable from getting tangled in obstacles. 
%
\revB{Since a real UAV experiment will involve significant amount of investment in terms of equipments, time and development of low-level controllers, we believe an experiment as that is outside the scope of the current paper and will require significant deviation from the theoretical \& algorithmic focus of the current paper. We thus propose such an experiment as a future work. But we believe that}
the new algorithmic tools proposed in this paper will \revB{be helpful in solving this real} problem in 3-d.
}

\mysection{Acknowledgements}

The authors acknowledge the support of federal contracts FA9550-12-1-0416 and FA9550-09-1-0643.
The first author acknowledges the support of ONR grant number N00014-14-1-0510 and University of Pennsylvania subaward number 564436.

}


\bibliographystyle{plainnat}
\bibliography{knot_planning}

\begin{thebibliography}{29}
\providecommand{\natexlab}[1]{#1}
\providecommand{\url}[1]{\texttt{#1}}
\expandafter\ifx\csname urlstyle\endcsname\relax
  \providecommand{\doi}[1]{doi: #1}\else
  \providecommand{\doi}{doi: \begingroup \urlstyle{rm}\Url}\fi

\bibitem[Arslan et~al.(2016)Arslan, Guralnik, and Koditschek]{Arslan:Tro:2016}
O.~Arslan, D.~P. Guralnik, and D.~E. Koditschek.
\newblock Coordinated robot navigation via hierarchical clustering.
\newblock \emph{IEEE Transactions on Robotics}, 32\penalty0 (2):\penalty0
  352--371, April 2016.
\newblock ISSN 1552-3098.

\bibitem[Bhattacharya et~al.(2015)Bhattacharya, Kim, Heidarsson, Sukhatme, and
  Kumar]{cable:separation:IJRR:14}
S.~Bhattacharya, S.~Kim, H.~Heidarsson, G.~Sukhatme, and V.~Kumar.
\newblock A topological approach to using cables to separate and manipulate
  sets of objects.
\newblock \emph{International Journal of Robotics Research}, online first
  publication, February 2015.
\newblock DOI: 10.1177/0278364914562236.

\bibitem[Bhattacharya and Ghrist(2015)]{Homotopy-Planning:IMAMR:15}
Subhrajit Bhattacharya and Robert Ghrist.
\newblock Path homotopy invariants and their application to optimal trajectory
  planning.
\newblock In \emph{Proceedings of IMA Conference on Mathematics of Robotics
  (IMAMR)}, St Anne's College, University of Oxford, September 9-11 2015.

\bibitem[Bhattacharya et~al.(2012)Bhattacharya, Likhachev, and
  Kumar]{planning:AURO:12}
Subhrajit Bhattacharya, Maxim Likhachev, and Vijay Kumar.
\newblock Topological constraints in search-based robot path planning.
\newblock \emph{Autonomous Robots}, pages 1--18, June 2012.
\newblock ISSN 0929-5593.
\newblock DOI: 10.1007/s10514-012-9304-1.

\bibitem[Bhattacharya et~al.(2013)Bhattacharya, Lipsky, Ghrist, and
  Kumar]{homology:amai:13}
Subhrajit Bhattacharya, David Lipsky, Robert Ghrist, and Vijay Kumar.
\newblock Invariants for homology classes with application to optimal search
  and planning problem in robotics.
\newblock \emph{Annals of Mathematics and Artificial Intelligence (AMAI)},
  67\penalty0 (3):\penalty0 251--281, March 2013.
\newblock DOI: 10.1007/s10472-013-9357-7.

\bibitem[Canny and Reif(1987)]{CRNew87}
J.~Canny and J.~H. Reif.
\newblock New lower bound techniques for robot motion planning problems.
\newblock In \emph{Proc. 28th Annu. IEEE Sympos. Found. Comput. Sci.}, pages
  49--60, 1987.

\bibitem[{Cohen} and {Pruidze}(2007)]{cohen:tori}
D.~C. {Cohen} and G.~{Pruidze}.
\newblock {Motion planning in tori}.
\newblock \emph{ArXiv Mathematics e-prints}, March 2007.

\bibitem[Cormen et~al.(2001)Cormen, Leiserson, Rivest, and Stein]{cormen2001}
T.~H. Cormen, C.~E. Leiserson, R.~L. Rivest, and C.~Stein.
\newblock \emph{Introduction to algorithms}.
\newblock MIT Press, 2nd edition, 2001.

\bibitem[Costa and Farber(2010)]{COSTA:planning}
Armindo Costa and Michael Farber.
\newblock Motion planning in spaces with small fundamental groups.
\newblock \emph{Communications in Contemporary Mathematics}, 12\penalty0
  (01):\penalty0 107--119, 2010.
\newblock \doi{10.1142/S0219199710003750}.

\bibitem[Crowell(1959)]{Crowell:Kampen:59}
Richard~H. Crowell.
\newblock On the van kampen theorem.
\newblock \emph{Pacific J. Math.}, 9\penalty0 (1):\penalty0 43--50, 1959.

\bibitem[Demyen and Buro(2006)]{Buro-TRAstar}
Douglas Demyen and Michael Buro.
\newblock Efficient triangulation-based pathfinding.
\newblock In \emph{AAAI'06: Proceedings of the 21st national conference on
  Artificial intelligence}, pages 942--947. AAAI Press, 2006.
\newblock ISBN 978-1-57735-281-5.

\bibitem[Epstein(1992)]{epstein1992word}
D.~B.~A. Epstein.
\newblock \emph{Word Processing in Groups}.
\newblock Ak Peters Series. Taylor \& Francis, 1992.

\bibitem[{Farber}(2001)]{Farber:top:complexity}
M.~{Farber}.
\newblock {Topological complexity of motion planning}.
\newblock \emph{ArXiv Mathematics e-prints}, November 2001.

\bibitem[{Farber} et~al.(2006){Farber}, {Grant}, and
  {Yuzvinsky}]{Farber:moving:obstacle}
M.~{Farber}, M.~{Grant}, and S.~{Yuzvinsky}.
\newblock {Topological complexity of collision free motion planning algorithms
  in the presence of multiple moving obstacles}.
\newblock \emph{ArXiv Mathematics e-prints}, September 2006.

\bibitem[Ghrist and LaValle(2006)]{GL:2006}
R.~Ghrist and S.~LaValle.
\newblock Nonpositive curvature and pareto optimal motion planning.
\newblock \emph{SIAM Journal of Control and Optimization}, 45\penalty0
  (5):\penalty0 1697--1713, 2006.

\bibitem[Govindarajan et~al.(2014)Govindarajan, Bhattacharya, and
  Kumar]{DARS:14:HRI}
Vijay Govindarajan, Subhrajit Bhattacharya, and Vijay Kumar.
\newblock Human-robot collaborative topological exploration for search and
  rescue applications.
\newblock In \emph{International Symposium on Distributed Autonomous Robotic
  Systems (DARS)}, 2014.

\bibitem[Greendlinger and Greendlinger(1986)]{greendlinger:dehn:1986}
E.~Greendlinger and M.~Greendlinger.
\newblock On dehn presentations and dehn algorithms.
\newblock \emph{Illinois J. Math.}, 30\penalty0 (2):\penalty0 360--363, 06
  1986.

\bibitem[Grigoriev and Slissenko(1998)]{Homotopy:Grigoriev:98}
D.~Grigoriev and A.~Slissenko.
\newblock Polytime algorithm for the shortest path in a homotopy class amidst
  semi-algebraic obstacles in the plane.
\newblock In \emph{ISSAC '98: Proceedings of the 1998 international symposium
  on Symbolic and algebraic computation}, pages 17--24, New York, NY, USA,
  1998. ACM.

\bibitem[Hatcher(2001)]{Hatcher:AlgTop}
Allen Hatcher.
\newblock \emph{Algebraic Topology}.
\newblock Cambridge Univ. Press, 2001.

\bibitem[Hershberger and Snoeyink(1991)]{Hershberger91computingminimum}
J.~Hershberger and J.~Snoeyink.
\newblock Computing minimum length paths of a given homotopy class.
\newblock \emph{Comput. Geom. Theory Appl}, 4:\penalty0 331--342, 1991.

\bibitem[Kim et~al.(2014)Kim, Bhattacharya, and Kumar]{ICRA:14:tethered}
S.~Kim, S.~Bhattacharya, and V.~Kumar.
\newblock Path planning for a tethered mobile robot.
\newblock In \emph{Proceedings of IEEE International Conference on Robotics and
  Automation}, Hong Kong, China, May 31 - June 7 2014.

\bibitem[Lickorish(1997)]{lickorish1997introduction}
W.B.R. Lickorish.
\newblock \emph{An Introduction to Knot Theory}.
\newblock Graduate Texts in Mathematics. Springer New York, 1997.
\newblock ISBN 9780387982540.

\bibitem[Lyndon and Schupp(2001)]{lyndon2001combinatorial}
R.C. Lyndon and P.E. Schupp.
\newblock \emph{Combinatorial Group Theory}.
\newblock Classics in Mathematics. Springer Berlin Heidelberg, 2001.
\newblock ISBN 9783540411581.

\bibitem[Mitchell and Sharir(2004)]{MSNew04}
Joseph S.~B. Mitchell and Micha Sharir.
\newblock New results on shortest paths in three dimensions.
\newblock In \emph{Proceedings of the Twentieth Annual Symposium on
  Computational Geometry}, pages 124--133. ACM, 2004.

\bibitem[Niblo and Reeves(1998)]{NRgeometry98}
G.~A. Niblo and L.~D. Reeves.
\newblock The geometry of cube complexes and the complexity of their
  fundamental groups.
\newblock \emph{Topology}, 37\penalty0 (3):\penalty0 621--633, 1998.

\bibitem[Schmitzberger et~al.(2002)Schmitzberger, Bouchet, Dufaut, Wolf, and
  Husson]{homotopy:Schmitzberger:02}
E.~Schmitzberger, J.L. Bouchet, M.~Dufaut, D.~Wolf, and R.~Husson.
\newblock Capture of homotopy classes with probabilistic road map.
\newblock In \emph{International Conference on Intelligent Robots and Systems},
  volume~3, pages 2317--2322, 2002.

\bibitem[Scott and Scott(1964)]{scott1964group}
W.R. Scott and W.R. Scott.
\newblock \emph{Group Theory}.
\newblock Dover Books on Mathematics Series. Dover Publ., 1964.
\newblock ISBN 9780486653778.

\bibitem[Tovar et~al.(2008)Tovar, Cohen, and LaValle]{tovar:sensor08}
Benjamín Tovar, Fred Cohen, and Steven~M. LaValle.
\newblock Sensor beams, obstacles, and possible paths.
\newblock In \emph{Workshop on the Algorithmic Foundations of Robotics}, pages
  317--332, 2008.

\bibitem[Weinbaum(1971)]{Weinbaum:knot:word:71}
C.~M. Weinbaum.
\newblock The word and conjugacy problems for the knot group of any tame,
  prime, alternating knot.
\newblock \emph{Proceedings of the American Mathematical Society}, 30\penalty0
  (1):\penalty0 22--26, September 1971.

\end{thebibliography}

\end{document}